\documentclass[pdflatex,sn-mathphys-num]{sn-jnl}

\usepackage{setspace}
\usepackage{tikz}
\usetikzlibrary{arrows.meta,calc}

\usepackage{graphicx}
\usepackage{multirow}
\usepackage{amsmath,amssymb,amsfonts}
\usepackage{amsthm}
\usepackage[title]{appendix}
\usepackage{xcolor}
\usepackage{textcomp}
\usepackage{manyfoot}
\usepackage{booktabs}
\usepackage{xspace}
\usepackage{algorithm}
\usepackage{algorithmicx}
\usepackage{algpseudocode}
\usepackage{listings}
\usepackage{enumerate}

\usepackage{bookmark}

\makeatletter
\renewcommand\normalsize{%
   \@setfontsize\normalsize{11bp}{13.6bp}%
   \abovedisplayskip 12\p@ \@plus2\p@ \@minus1\p@
   \abovedisplayshortskip \z@ \@plus3\p@%
   \belowdisplayshortskip 3\p@ \@plus3\p@ \@minus3\p@%
   \belowdisplayskip \abovedisplayskip%
   \let\@listi\@listI}%
\normalsize
\renewcommand\small{%
   \@setfontsize\small{10bp}{12bp}%
   \abovedisplayskip 5\p@ \@plus3\p@ \@minus4\p@
   \abovedisplayshortskip \z@ \@plus2\p@
   \belowdisplayshortskip 3\p@ \@plus2\p@ \@minus2\p@
   \belowdisplayskip \abovedisplayskip}%
\renewcommand\footnotesize{%
   \@setfontsize\footnotesize{8bp}{9.5bp}%
   \abovedisplayskip 5\p@ \@plus2\p@ \@minus4\p@
   \abovedisplayshortskip \z@ \@plus\p@
   \belowdisplayshortskip 3\p@ \@plus\p@ \@minus2\p@
   \belowdisplayskip \abovedisplayskip}
\makeatother

\newtheorem{theorem}{Theorem}[section]

\theoremstyle{thmstyleone}

\newtheorem{corollary}[theorem]{Corollary}

\theoremstyle{thmstyletwo}

\theoremstyle{thmstylethree}

\newcommand{\STAMP}{\texorpdfstring{{$\textrm{STAMP}$}}{STAMP}\xspace}
\begin{document}

\title[Article Title]{STAMP: Symmetric Temporal Averaging Message Passing for Domain Adaptation in Chronologically Split Graphs}

\author[1]{\fnm{Joo Young} \sur{Park}}\email{jyp531@snu.ac.kr}
\equalcont{These authors contributed equally to this work.}

\author[2]{\fnm{Sejun} \sur{Park}}\email{sejun3@illinois.edu}
\equalcont{These authors contributed equally to this work.}

\author*[3]{\fnm{Hyunwoo} \sur{Park}}\email{hyunwoopark@snu.ac.kr}

\affil[1]{\footnotesize \orgdiv{College of Medicine}, \orgname{Seoul National University}, \orgaddress{\street{103 Daehak-ro}, \city{Seoul}, \postcode{03080}, \country{Republic of Korea}}}

\affil[2]{\footnotesize \orgdiv{Siebel School of Computing and Data Science}, \orgname{University of Illinois Urbana-Champaign}, \orgaddress{\street{201 North Goodwin Avenue}, \city{Urbana}, \postcode{61801}, \state{Illinois}, \country{United States of America}}}

\affil*[3]{\footnotesize \orgdiv{Graduate School of Data Science}, \orgname{Seoul National University}, \orgaddress{\street{1 Gwanak-ro, Gwanak-gu}, \city{Seoul}, \postcode{08826}, \country{Republic of Korea}}}

\abstract{This paper addresses domain adaptation in graph structured data under chronological splits. We consider a transductive graph learning setting in which each node is associated with a timestamp, and focus on Semi-Supervised Node Classification (SSNC), where labels of historical nodes are used to predict labels of more recent nodes. In such settings, temporal dependencies in graph connectivity induce distribution shifts across time, often causing substantial performance degradation when models trained on past nodes are deployed on future nodes. Despite its practical relevance, domain adaptation under chronological graph splits remains relatively underexplored. We propose Symmetric Temporal Averaging Message Passing (\STAMP{}), a timestamp-balanced message passing operator that replaces standard neighborhood aggregation with a temporally aligned propagation scheme. We refer to the specialized variant for the heterogeneous ogbn-mag dataset as STAMP-mag, while reserving \STAMP{} for the general formulation applicable to arbitrary graph datasets. By balancing messages across timestamps, \STAMP{} enforces first-moment invariance across temporal domains while remaining scalable and easy to implement. Unlike many existing domain adaptation methods, whose assumptions are difficult to verify in graph settings, \STAMP{} relies on transparent and practically checkable conditions, and is further supported by theoretical analysis. Empirically, \STAMP{} and STAMP-mag improve accuracy across all evaluated temporal graph benchmarks. \STAMP{} achieves state-of-the-art (SOTA) performance on ogbn-arxiv, while STAMP-mag does so on ogbn-mag. For temporal-split Pokec, where no established benchmark exists, applying \STAMP{} to the state-of-the-art model for the original dataset consistently improves performance under our chronological split. These results demonstrate that timestamp-balanced propagation is an effective and practical approach for mitigating temporal distribution shift in graph learning.}


\maketitle

{%
\renewcommand{\thefootnote}{}%
\footnotetext{\footnotesize $^{\dagger}$These authors contributed equally to this work. E-mail(s): \textcolor{blue}{jyp531@snu.ac.kr}; \textcolor{blue}{sejun3@illinois.edu}}%
\footnotetext{\footnotesize *Corresponding author. E-mail: \textcolor{blue}{hyunwoopark@snu.ac.kr}}%
}

\newpage

\section{Introduction}\label{sec1_Introduction}

Semi-supervised node classification (SSNC) on graphs frequently arises in settings where nodes are associated with temporal information. For example, in an academic citation network, each node representing a paper may contain information including the year of its publication. Beyond citation networks, chronological graph structures arise in a wide range of practical applications, including fraud detection in temporal financial transaction networks \cite{xiang2023semi,duan2024cat}, evolving social and communication networks \cite{longa2023graph,yang2024dynamic,zheng2025survey}, and dynamic recommendation systems \cite{bai2020temporal,gao2023survey,kim2024content}. In these settings, the ability to accurately classify recently added nodes using models trained on historical data is of paramount importance. This work focuses on such temporally annotated graph data, particularly on datasets where the train and test splits are arranged in chronological order. More specifically, as illustrated in Figure~\ref{fig:chronological_split}, models are required to infer labels for the most recent nodes while trained solely from historical observations.

\begin{figure}[hbt!]
    \centering
	\includegraphics[width=0.90\textwidth]{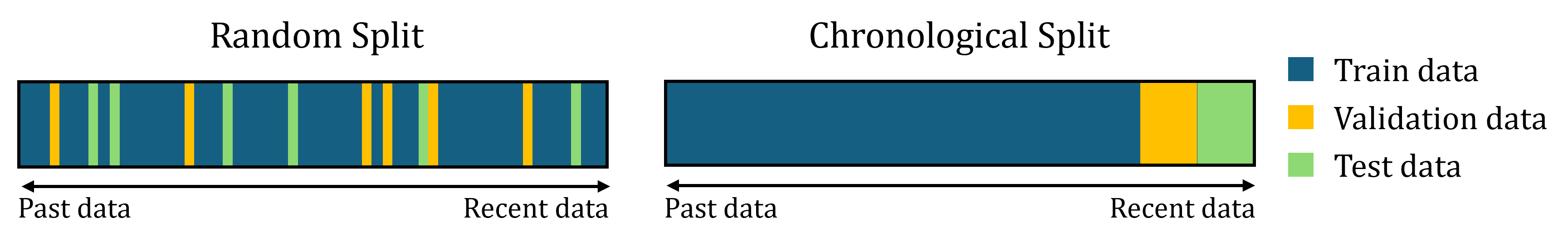}
	\caption{Illustrative explanation of a chronological split dataset. Nodes are organized by their timestamps along the horizontal axis. The training set consists of nodes from earlier time periods, while the test set contains the most recent nodes. Edges may connect nodes across different timestamps, and the neighborhood composition of a node can change over time, leading to distribution shift between the training and test domains.}
	\label{fig:chronological_split}
\end{figure}

Although training graph neural networks (GNNs) using past data to classify recent data is a common scenario in both industrial applications and research practice \cite{liu2016predicting,bai2020temporal,pareja2020evolvegcn}, systematic research on the effective use of temporal information within chronologically split graphs remains limited. Failure to adequately utilize temporal information can lead to significant performance degradation when the model attempts to classify the labels of recent data. This is because in temporal graphs, the neighborhood distribution of a target node is not stationary: both the composition of its neighboring nodes and the frequency of their interactions may evolve over time.

In this work, we propose Symmetric Temporal Averaging Message Passing (\STAMP{}), a timestamp-balanced message passing operator to mitigate the temporal distribution shift in chronologically split graphs. \STAMP{} performs aggregation in a balanced manner across timestamps, which enforces first-moment invariance across temporal domains at the initial message-passing level. As a result, \STAMP{} directly targets the distributional discrepancy induced by chronological splits, while remaining scalable and easy to incorporate into existing GNN architectures by replacing a linear feature propagation layer. We further provide theoretical analysis showing that \STAMP{} removes the first-moment component of the train-test representation shift and improves the Wasserstein-based distribution-shift term between train and test domains, and also achieves the minimax optimality. Finally, \STAMP{} improves accuracy on ogbn-arxiv and temporal-split Pokec, achieving SOTA on ogbn-arxiv and outperforming the original Pokec SOTA baseline under the chronological split. We further introduce STAMP-mag, a specialized variant designed for the heterogeneous ogbn-mag dataset, which achieves SOTA performance on ogbn-mag.

\section{Related Work}\label{sec2_Related Work}

\subsection{Graph Neural Networks}
\label{subsec2_GNN}

Graph Neural Networks (GNNs) have gained significant attention across a wide range of domains, including recommendation systems \cite{ying2018graph,gao2022graph}, biology \cite{barabasi2004network}, and chemistry \cite{wu2018moleculenet}. Spatial GNN architectures, such as Graph Convolutional Networks (GCN) \cite{gcn}, Graph Sample and Aggregate (GraphSAGE) \cite{Graphsage}, Graph Attention Networks (GAT) \cite{velickovic2017graph}, and Hypergraph Neural Networks (HGNN) \cite{HGNN}, learn node representations by iteratively aggregating information from neighboring nodes through message passing. This general mechanism can be written as
\begin{align}
    M_v^{(k+1)} 
    &\leftarrow 
    \operatorname{AGG}\left(\left\{X_u^{(k)}:u\in \mathcal{N}_v\right\}\right),
    \qquad \forall v\in \mathbf{V}, \ 0\le k<K,\\
    X_v^{(k+1)}
    &\leftarrow
    \operatorname{COMBINE}\left(X_v^{(k)},M_v^{(k+1)}\right),
    \qquad \forall v\in \mathbf{V}, \ 0\le k<K.
\end{align}
Here, $\mathbf{V}$ denotes the set of nodes, $K$ is the number of GNN layers, and $X_v^{(k)}$ is the representation of node $v$ at the $k$-th layer, with $X_v^{(0)}=X_v$ corresponding to the initial node feature. The aggregation function $\operatorname{AGG}$ collects information from neighboring nodes, while $\operatorname{COMBINE}$ updates the node representation by combining the aggregated message with the previous representation. The final representation $X_v^{(K)}$ is then used for downstream prediction, such as node classification.

When GNNs are applied to large-scale graphs, scalability becomes a central challenge. The receptive field of a node, often referred to as its ego graph, grows rapidly with the number of message-passing layers. Therefore, scalable GNN architectures must be designed to efficiently control the cost of neighborhood expansion \cite{Graphsage,clustergnn,zenggraphsaint}. Decoupled GNN frameworks, including Simple Graph Convolution (SGC) \cite{SGC}, Scalable Inception GNN (SIGN) \cite{rossi2020sign}, and Graph Attention Multi-Layer Perceptron (GAMLP) \cite{GAMLP}, address this issue by separating feature propagation from model training. These methods precompute neighborhood aggregation in a parameter-free manner, significantly improving efficiency while maintaining strong predictive performance on large-scale graphs. For heterogeneous graphs, Simple and Efficient Heterogeneous Graph Neural Network (SeHGNN) \cite{SeHGNN}, Random Projection Heterogeneous GNN (RpHGNN) \cite{RpHGNN}, and more recent high-performing heterogeneous graph methods such as PSHGCN \cite{he2023pshgcn} and HGAMLP-HOPE \cite{zhou2026scalable} further demonstrate the effectiveness of scalable propagation-based architectures.

Recent empirical studies also suggest that carefully tuned classical GNNs remain highly competitive. In particular, Luo et al.~\cite{luo2024classic} reassessed GCN, GAT, and GraphSAGE against recent graph transformer models and showed that, with proper hyperparameter tuning, classical message-passing GNNs can match or outperform more complex architectures on many node-classification benchmarks, including Pokec and ogbn-arxiv.

\subsection{Temporal Graph Learning and Graph Distribution Shift}
\label{subsec2_temporal_ood}

Our setting is graph learning under chronological splits, where nodes contain temporal information rather than edges. Such splits induce distribution shifts, and benchmarks such as GOOD \cite{gui2022good} and recent surveys \cite{liu2024beyond,zhang2024surveyDeepGraphDistributionShifts} show that GNNs can suffer substantial performance degradation. Recent graph domain adaptation methods, such as SOGA \cite{mao2024source}, GraphCTA \cite{zhang2024collaborate}, and DREAM \cite{yin2024dream}, aim to learn transferable representations across domains. More recently, DGSDA \cite{yang2025disentangled} addressed differences in node features and graph structures between domains, while GCAL \cite{qiao2025gcal} focused on adapting to new graph domains without forgetting previously learned information. Lei et al.~\cite{lei2026gradual} proposed gradually adapting between different graph domains using intermediate graphs, while Lin et al.~\cite{lin2024scalable} studied pretraining GNNs on multiple graphs for transfer to new domains. PromptDyG \cite{ai2026promptdyg} further addressed changes in graph structure over time by adapting pretrained models to newly observed graphs without requiring additional labels.

Domain adaptation has also been studied in other fields. In computer vision, Zhu et al.~\cite{zhu2025revisiting} studied how to adapt image classification models to new datasets without access to the original training data. In multi-agent collaborative perception, CoPEFT \cite{wei2025copeft} proposed adapting pretrained models to new environments by updating only a small fraction of their parameters. Unlike these approaches, which rely on representation alignment or model adaptation, \STAMP{} and STAMP-mag directly modify the linear message-passing operator to balance neighborhood aggregation across timestamps, without requiring additional adaptation objectives or fine-tuning. This allows \STAMP{} and STAMP-mag to be incorporated into strong contemporary baselines such as tuned GCNs on Pokec \cite{luo2024classic}, SimTeG+TAPE+RevGAT on ogbn-arxiv \cite{RevGAT,Duan2023SimTeG}, and HGAMLP-HOPE on ogbn-mag \cite{zhou2026scalable}.

\subsection{Theoretical Analysis for Domain Adaptation} \label{subsec2_DA}
During a deep learning process, models are trained on data drawn from a training domain and evaluated on a test domain. Domain adaptation addresses scenarios in which training and test data are generated from different distributions, and therefore strong performance on the training domain does not necessarily guarantee comparable performance on the test domain \cite{pan2009survey, ben2010theory}. The performance on the test domain is known to depend both on the performance on the training domain and on the discrepancy between the two domains \cite{ben2006analysis, ben2010theory, germain2013pac, germain2016new}.

Let $\mathcal{X}$ denote the feature space and $\mathcal{Y}$ the label space. The learning objective is to identify a predictor $f : \mathcal{X} \rightarrow \mathcal{Y}$ that minimizes the training risk $R_{tr}(f) = \mathbb{E}_{(X,Y) \sim P_{tr}} \left[ L(f(X), Y) \right] $ where $P_{tr}$ denotes the joint distribution of training feature-label pairs and $L : \mathcal{Y} \times \mathcal{Y} \rightarrow \mathbb{R}$ is a loss function. However, the ultimate goal is to minimize the test risk $R_{te}(f) = \mathbb{E}_{(X,Y) \sim P_{te}} \left[ L(f(X), Y) \right]$ where $P_{te}$ represents the joint distribution of the test feature-label pairs. Recently, theories and applications adapted the Wasserstein-1 distance ($W_1$) for the metric between the test and train domain distributions \cite{lee2019sliced, shen2018wasserstein, arjovsky2017wasserstein, redko2017theoretical}. Under appropriate regularity conditions (see \cite{redko2017theoretical}), the test risk admits the following upper bound:
\begin{align}
    R_{te}(f) \le R_{tr}(f) + W_1(Q_{tr}, Q_{te}) + \lambda ,
\end{align}
where $Q_{tr}$ and $Q_{te}$ denote the marginal feature distributions of $P_{tr}$ and $P_{te}$, respectively, and $\lambda$ accounts for the joint error of the ideal hypothesis across the two domains.

\section{Method} \label{sec3_Method}

\subsection{\STAMP{}: Timestamp-Balanced Message Passing} \label{subsec3_Framework}
Denote the graph dataset as $\mathbf{G} = (\mathbf{V}, \mathbf{E})$, and denote the set of possible timestamps by $\mathbf{T} = \{1, 2, ..., t_{max} \}$ and the set of labels by $\mathbf{Y}$. The training data consist of historical nodes, namely those with timestamps strictly smaller than $t_{max}$, while the test data consist of nodes with timestamp $t_{max}$. Consequently, the labels of timestamp $t_{max}$ nodes are unknown. Throughout the paper, $y$ and $t$ denote the label and timestamp of a target node, while $\tilde{y}$ and $\tilde{t}$ denote the label and timestamp of one of its neighboring nodes, unless otherwise specified.

Message passing refers to the procedure of aggregating representations from neighboring nodes, typically applied across multiple layers. We begin by introducing the widely used averaging message passing scheme \cite{Graphsage, gcn}. For an arbitrary target node $v \in \mathbf{V}$ with label $y$ and time $t$,
\begin{equation}
    M_{v}^{(k+1)}= \frac{\sum_{\tilde{y}\in \mathbf{Y}}\sum_{\tilde{t}\in\mathbf{T}}\sum_{w\in\mathcal{N}_{v}(\tilde{y}, \tilde{t})}X_w^{(k)}}{{\sum_{\tilde{y}\in \mathbf{Y}}\sum_{\tilde{t}\in\mathbf{T}} |\mathcal{N}_{v}(\tilde{y}, \tilde{t})|}} ,
\end{equation}
where $\mathcal{N}_{v}(\tilde{y}, \tilde{t}) := \{w \in \mathbf{V} \mid w \; \text{is a neighbor of} \; v \; \text{with label} \; \tilde{y} \; \text{and time} \; \tilde{t} \}$, $X_w^{(k)}$ is the $k$-th representation considered as a random variable, and $M_v^{(k+1)}$ is the aggregated message at node $v$ in the $(k+1)$-th layer. Let $\mathcal{N}_v$ denote the full neighborhood of $v$. Then, $|\mathcal{N}_{v}| = \sum_{\tilde{y}\in \mathbf{Y}} \sum_{\tilde{t}\in\mathbf{T}} |\mathcal{N}_{v}(\tilde{y}, \tilde{t})|$.

Suppose the initial features $X_w^{(0)}$ have means which depends on time, i.e., ${\mathbb{E}}[X_w^{(0)}]=m_1(\tilde{t})$ for $\forall w \in \mathcal{N}_{v}\left(\tilde{y}, \tilde{t} \right)$, where $m_1: \mathbf{T} \rightarrow \mathbb{R}^f$ and $f$ is the dimension of the representation space.

We first examine the first moment of the aggregated message in the initial layer:
\begin{align}
    {\mathbb E}\left[{M_{v}^{(1)}}\right] &= {\mathbb E}\left[\frac{\sum_{\tilde{y}\in \mathbf{Y}}\sum_{\tilde{t}\in\mathbf{T}}\sum_{w\in\mathcal{N}_{v}(\tilde{y}, \tilde{t})}X_w^{(0)}}{\sum_{\tilde{y}\in \mathbf{Y}}\sum_{\tilde{t}\in\mathbf{T}} |\mathcal{N}_{v}(\tilde{y}, \tilde{t})|}\right] \nonumber\\
    &= \mathbb{E}\left[\frac{\sum_{\tilde{y}\in \mathbf{Y}}\sum_{\tilde{t}\in\mathbf{T}}\sum_{w\in\mathcal{N}_{v}(\tilde{y}, \tilde{t})}X_w^{(0)}}{|\mathcal{N}_{v}|}\right] \nonumber\\
    &= \frac{1}{|\mathcal{N}_{v}|}\sum_{\tilde{y}\in \mathbf{Y}}\sum_{\tilde{t}\in\mathbf{T}}\sum_{w\in\mathcal{N}_{v}(\tilde{y}, \tilde{t})}m_1(\tilde{t}) \nonumber\\
    &= \frac{1}{|\mathcal{N}_{v}|}\sum_{\tilde{y}\in \mathbf{Y}}\sum_{\tilde{t}\in\mathbf{T}}|\mathcal{N}_v(\tilde{y}, \tilde{t})|m_1(\tilde{t}) ,
\end{align}
which depends on label and time of the target node $v$ since $|\mathcal{N}_v(\tilde{y}, \tilde{t})|/|\mathcal{N}_{v}|$ is dependent on $y$ and $t$. The \STAMP{} is defined in order to align the first moment of the initial aggregated messages to be invariant across time and labels as follows :
\begin{align}
    M_{v}^{(1)} &:= \frac{1}{|\mathbf{T}|} \sum_{\tilde{y}\in \mathbf{Y}}\sum_{\tilde{t}\in \mathbf{T}}\sum_{w\in\mathcal{N}_{v}(\tilde y, \tilde t)} \frac{1}{\left| \mathcal{N}_{v} (\tilde{t}) \right|}X_w ^{(0)} \nonumber \\
    &= \frac{1}{|\mathbf{T}|} \sum_{\tilde{t}\in \mathbf{T}}\sum_{w\in\mathcal{N}_{v}(\tilde t)} \frac{1}{\left| \mathcal{N}_{v} (\tilde{t}) \right|}X_w ^{(0)} ,
\end{align}
where $\mathcal{N}_{v} (\tilde{t}) := \{w \in \mathbf{V} \mid w \; \text{is a neighbor of} \; v \; \text{with time} \; \tilde{t} \}$. We additionally assume that each target node has at least one neighbor at every timestamp, i.e.,
\begin{equation}
    |\mathcal N_v(\tilde t)|>0,
    \qquad
    \forall v\in\mathbf V,\ \forall \tilde t\in\mathbf T .
\end{equation}
This assumption ensures that the \STAMP{} aggregation operator is well-defined. In large-scale graphs with sufficiently dense cross-time connectivity, this condition is often mild and can be empirically verified. Now, we can observe that
\begin{equation}
    {\mathbb E}[{M_{v}^{(1)}}] = \frac{1}{|\mathbf{T}|} {\sum_{\tilde{t}\in \mathbf{T}}} m_1(\tilde{t}) ,
\end{equation}
and therefore the mean of $M_v^{(1)}$ is invariant with time $t$ and label {y}.

Below, we provide a figure illustrating the overall framework of \STAMP{}, and a pseudocode of \STAMP{}. We further provide empirical validations of the assumption ${\mathbb{E}}[X_w^{(0)}]=m_1(\tilde{t})$ for $\forall w \in \mathcal{N}_{v}\left(\tilde{y}, \tilde{t} \right)$ for the datasets ogbn-arxiv and temporal-split Pokec, in Appendix \ref{apdx:assumption verifications}.

\begin{figure*}[t]
    \centering
    \includegraphics[width=\textwidth]{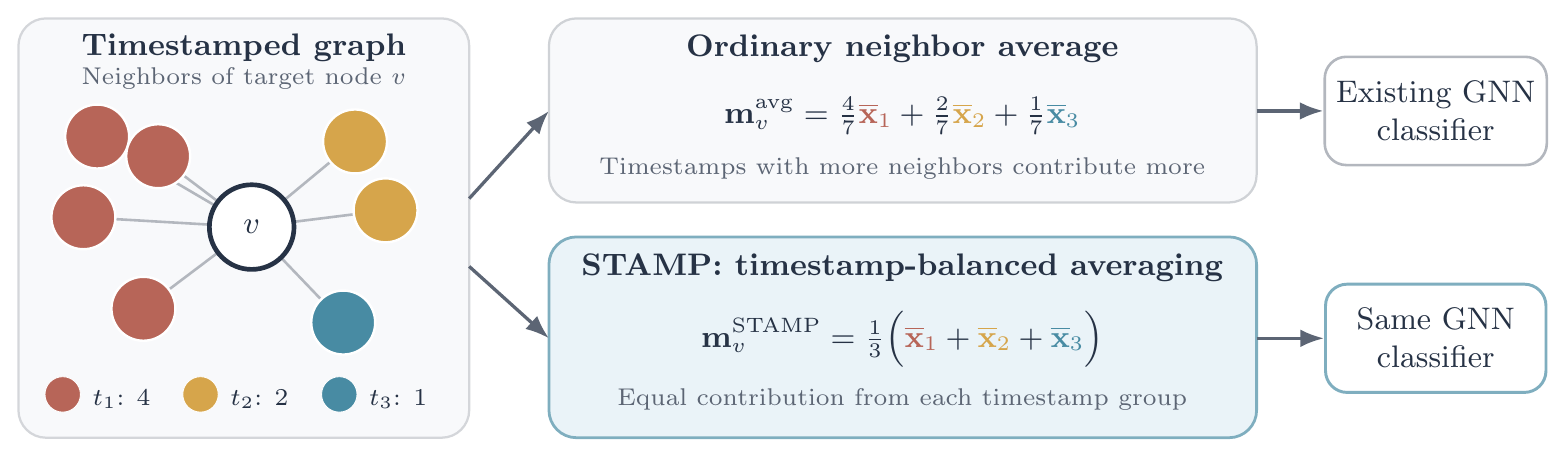}
    \caption{Overview of STAMP. Conventional message passing assigns greater weight to timestamps with more neighbors, whereas STAMP balances contributions across timestamps before passing the aggregated features to the GNN.}
    \label{fig:stamp_overview}
\end{figure*}

\begin{algorithm}
\small
\caption{Symmetric Temporal Averaging Message Passing (\STAMP{})}
\label{alg:stamp}
\begin{algorithmic}[1]
\Require Neighbor sets $\mathcal N_v,~\forall v\in\mathbf V$; node time function $time:\mathbf V\to\mathbf T$
\Ensure Directional aggregation coefficients $\{\alpha_{u\to v}\}_{u\in\mathcal N_v,\,v\in\mathbf V}$
\Statex

\For{$v\in\mathbf V$}
    \For{$\tau\in\mathbf T$}
        \State $C_v(\tau)\leftarrow |\mathcal N_v(\tau)|$
    \EndFor
    \For{$u\in\mathcal N_v$}
        \State $\alpha_{u\to v}\leftarrow
        \dfrac{1}{|\mathbf T|\,C_v(time(u))}$
    \EndFor
\EndFor
\Statex
\State Use the propagation rule
\[
    M_v^{(1)}
    =
    \sum_{u\in\mathcal N_v}
    \alpha_{u\to v}X_u^{(0)}.
\]
\end{algorithmic}
\end{algorithm}

\subsection{STAMP-mag: A Specialized Variant for ogbn-mag}
We propose STAMP-mag as an approach to achieve first moment invariance specifically for the dataset ogbn-mag. In this case we set the possible timestamps by $\mathbf{T}: = \{\dots, t_{max} - 2, t_{max} - 1, t_{max}\}$. Denote ${P}_{y t}(\tilde{y}, \tilde{t})$ as the probability distribution over the label–timestamp pairs of neighboring nodes. By definition, it satisfies $\sum_{\tilde{y} \in \mathbf{Y}}\sum_{\tilde{t} \in \mathbf{T}} {P}_{y t} (\tilde{y}, \tilde{t})=1$. Now, the following assumption gives a unique property which arises the use of STAMP-mag.

\begin{align}\label{assumption}
    &\textit{Assumption 1}: {P}_{y t} \left(\tilde{y}, \tilde{t}\right) = f(y, t) g\left(y, \tilde{y}, | \tilde{t}-t|\right), \ \forall y, \tilde{y} \in \mathbf{Y}, \forall t, \tilde{t} \in \mathbf{T} . \\
    &\textit{Assumption 2}: \mathbb E [X_w^{(0)}] = \mu_1(\tilde y), \ \forall w \ \text{with label} \ \tilde y .
\end{align}

\begin{figure}[hbt!]
    \centering
    \includegraphics[width=0.50\textwidth]{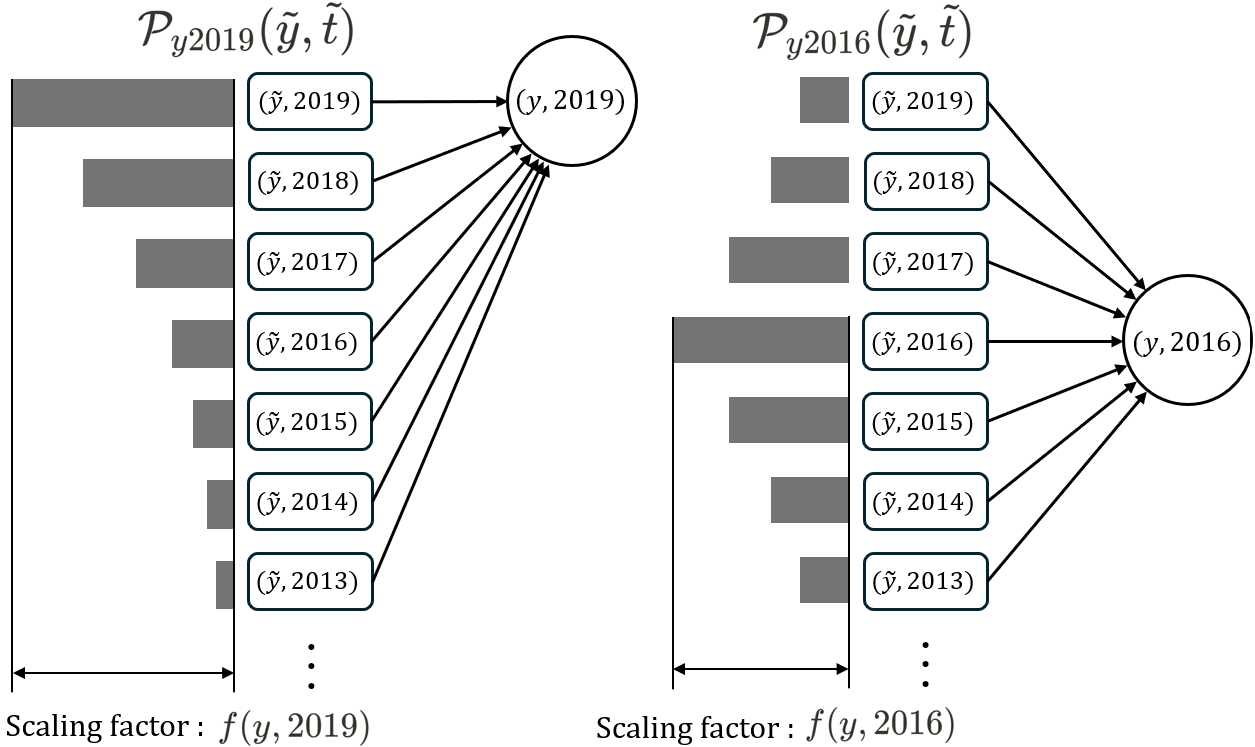}
    \caption{Graphical representation of functions $f$ and $g$, which jointly govern the distribution of neighboring node label--timestamp pairs under Assumption~3. The shaded bars denote relative connectivity for a target node with label $y$; only neighboring nodes with label $\tilde{y}$ are shown. The function $g(y,\tilde{y},|\tilde{t} - t|)$ captures how relative connectivity decays with temporal distance, while the normalization factor $f(y, t)$ adjusts the overall scale so that ${P}_{yt}(\tilde{y}, \tilde{t})$ forms a valid probability distribution.}
    \label{fig:assumptions of f and g}
\end{figure}

Assumption 1 states a separability condition on the distribution of neighboring nodes, motivated by the empirical observation that the proportion of neighboring nodes with timestamp $\tilde{t}$ for a target node at time $t$ decreases as the temporal distance $|\tilde{t} - t|$ increases in the ogbn-mag dataset. We modeled this behavior using a function $g\big(y, \tilde{y}, |\tilde{t} - t|\big)$, which characterizes the proportion of neighboring nodes as a function of temporal distance. In addition, a normalization function $f(y,t)$ is introduced to ensure that ${P}_{y t}(\tilde{y}, \tilde{t})$ forms a valid probability distribution. Figure~\ref{fig:assumptions of f and g} provides a graphical illustration of these properties. Assumption 2 states that the mean of the initial features are time-invariant, and only dependent to the label. These assumptions are motivated by structural patterns observed in the real-worl citation graph ogbn-mag, and we provide empirical validations in Appendix \ref{apdx:assumption verifications}.

Now, consider a target node $v$ with timestamp $t$, and let $\tilde{t}$ denote the timestamp of a neighboring node. Define $\Delta = |\tilde{t} - t|$. If $0<\Delta \le | t_{max}-t|$, then $\tilde{t}$ admits two symmetric candidates, namely $t + \Delta$ and $t - \Delta$. In contrast, if $\Delta > |t_{\max} - t|$ or $\Delta = 0$, the timestamp $\tilde{t}$ corresponds only to the single value $t - \Delta$. This asymmetry in temporal offsets induces unequal contributions during aggregation. To formalize this distinction, we partition the time index set as $\mathbf{T}_{t}^{\text{double}} = \{\tilde t \in \mathbf{T} \mid 0<|\tilde t - t| \le | t_{max}-t|\}$ and $\mathbf{T}_{t}^{\text{single}} = \{\tilde t \in \mathbf{T} \mid | \tilde t - t| > | t_{max}-t| \; \text{or} \; \tilde t = t\}$. Therefore, under conventional averaging message passing, nodes with timestamps in $\mathbf{T}_{t}^{\text{double}}$ contribute twice than the nodes with timestamps in $\mathbf{T}_{t}^{\text{single}}$. The key idea of STAMP-mag is to correct this imbalance by assigning doubled weights to neighbor nodes whose timestamps belong to $\mathbf{T}_{t}^{\text{single}}$, via graph modification in the preprocessing stage . From now on, we call the nodes with timestamp belonging to $\mathbf{T}_{t}^{\text{double}}$ and $\mathbf{T}_{t}^{\text{single}}$ as \emph{double nodes} and \emph{single nodes}, respectively.

\begin{figure}[hbt!]
    \centering
    \includegraphics[width=0.25\textwidth]{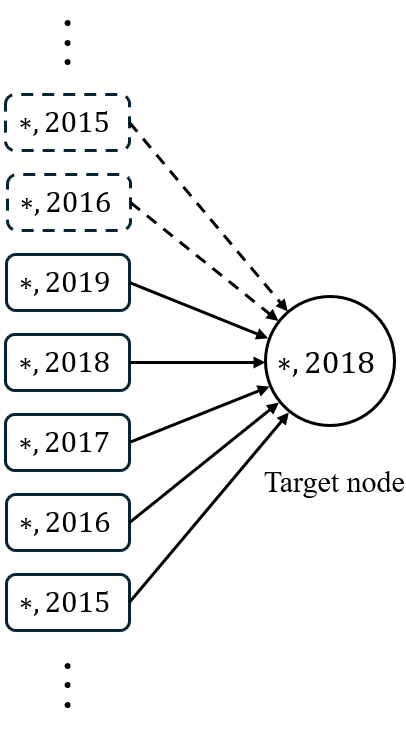}
    \caption{Graphical illustration of STAMP-mag. For a target node at timestamp $t=2018$, neighbors at timestamps $2017$ and $2019$ form a symmetric pair ($\Delta=1$), contributing equally during aggregation. Neighbors at timestamps outside this symmetric range (e.g., $2016$, $2015$, \ldots) are single nodes whose weights are doubled under STAMP-mag to restore balance.}
    \label{fig:SMP}
\end{figure}

To illustrate, consider Figure \ref{fig:SMP}, where the target node has timestamp $t=2018$. In this case, $\mathbf{T}_{t}^{\text{double}} = \{2019, 2017\}$ and $\mathbf{T}_{t}^{\text{single}} = \{2018, 2016, 2015, 2014, ...\}$. Neighbor nodes with timestamps $2017$ and $2019$ correspond to $\Delta = 1$ and form a symmetric pair around $t$. By assumption, these nodes contribute equally during message passing. However, single nodes lack such symmetric counterparts. Consequently, double nodes collectively contribute twice as much as single nodes do under conventional aggregation. We modify the initial graph during the preprocessing stage, in order to double the weight of single nodes. Accordingly, STAMP-mag adjusts the balance of double and single nodes to contribute equally when message passing, and therefore achieves first moment alignment for the initial message layers.

The STAMP-mag from the $k$-th layer to the $(k+1)$-th layer of target node $v$ is defined by
\begin{equation}
    M_{v}^{(k+1)} := \frac{{\sum_{\tilde{y}\in \mathbf{Y}}\sum_{\tilde{t}\in\mathbf{T}_{t}^{\text{single}}}\sum_{w\in\mathcal{N}_{v}(\tilde{y}, \tilde{t})}2X_w^{(k)}+\sum_{\tilde{y}\in \mathbf{Y}}\sum_{\tilde{t}\in\mathbf{T}_{t}^{\text{double}}}\sum_{w\in\mathcal{N}_{v}(\tilde{y}, \tilde{t})}X_w^{(k)}}}{\sum_{\tilde{y}\in \mathbf{Y}}\sum_{\tilde{t}\in\mathbf{T}_{t}^{\text{single}}}2|\mathcal{N}_{v}(\tilde{y}, \tilde{t})|+\sum_{\tilde{y}\in \mathbf{Y}}\sum_{\tilde{t}\in\mathbf{T}_{t}^{\text{double}}}|\mathcal{N}_{v}(\tilde{y}, \tilde{t})|} .
\end{equation}

Suppose the initial features $X_w^{(0)}$ have means which depends on time, i.e., ${\mathbb{E}}[X_w^{(0)}]=m_1(\tilde{t})$ for $\forall w \in \mathcal{N}_{v}\left(\tilde{y}, \tilde{t} \right)$, where $m_1: \mathbf{T} \rightarrow \mathbb{R}^f$ and $f$ is the dimension of the representation space. Then,
\begin{align} \label{first moment approx of stamp mag}
    {\mathbb E}\left[{M_{v}^{(1)}}\right] & \simeq \frac{\sum_{\tilde y\in \mathbf{Y}}\sum_{\tilde t\in\mathbf{T}_{t}^{\text{single}}}2P_{y t}(\tilde y, \tilde t) m_1(\tilde{t})+\sum_{\tilde y\in \mathbf{Y}}\sum_{\tilde t\in\mathbf{T}_{t}^{\text{double}}}P_{y t}(\tilde y, \tilde t) m_1(\tilde{t})}{\sum_{\tilde y\in \mathbf{Y}}\sum_{\tilde t\in\mathbf{T}_{t}^{\text{single}}}2P_{y t}(\tilde y, \tilde t)+\sum_{\tilde y\in \mathbf{Y}}\sum_{\tilde t\in\mathbf{T}_{t}^{\text{double}}}P_{y t}(\tilde y, \tilde t)} ,
\end{align}

\noindent using the approximation ${|\mathcal{N}_{v}(\tilde{y}, \tilde{t})| / |\mathcal{N}_{v}|} \simeq P_{yt} (\tilde{y}, \tilde{t})$. The key property is that the first moment of the initial messages obtained by STAMP-mag does not depend on the target node timestamp $t$ since by assumption as follows.
\begin{align}
    &\sum_{\tilde t\in\mathbf{T}_{t}^{\text{single}}}2P_{y t}(\tilde y, \tilde t)+\sum_{\tilde t\in\mathbf{T}_{t}^{\text{double}}}P_{y t}(\tilde y, \tilde t) \nonumber \\
    &=f(y, t )\big(\sum_{\tilde t\in\mathbf{T}_{t}^{\text{single}}}2g(y, \tilde y, |\tilde t - t|)+\sum_{\tilde t\in\mathbf{T}_{t}^{\text{double}}}g(y, \tilde y, |t - t|)\big)\nonumber \\
    &=f(y, t )\big(2g(y, \tilde y, 0)+2\sum_{\tau>|t_{max}-t |}g(y, \tilde y,\tau)+2\sum_{0<\tau\le|t_{max}-t|}g(y, \tilde y, \tau)\big) \nonumber \\
    &= 2f(y, t )\sum_{\tau\ge 0}g(y, \tilde y, \tau) ,
\end{align}
\noindent and plugging this into \ref{first moment approx of stamp mag}, we obtain
\begin{align}
    {\mathbb E}\left[{M_{v}^{(1)}}\right] &\simeq \frac{\sum_{\tilde y \in \mathbf{Y}} f(y, t )\sum_{\tau\ge 0}g(y, \tilde y, \tau) \mu_{X}^{(0)}(\tilde y )}{ \sum_{\tilde y \in \mathbf{Y}} f(y, t) \sum_{\tau \ge 0}g(y, \tilde y, \tau)} \nonumber \\
    &= \frac{\sum_{\tilde y \in \mathbf{Y}} \sum_{\tau\ge 0}g(y, \tilde y, \tau) \mu_{X}^{(k)}(\tilde y )}{\sum_{\tilde y \in \mathbf{Y}} \sum_{\tau \ge 0}g(y, \tilde y, \tau)} ,
\end{align}
\noindent which does not depend on $t$. Therefore, the first moment of the initial layer message is invariant with time. Below, we provide the pseudocode of STAMP-mag. As discussed, STAMP-mag modifies the graph structure in the preprocessing stage, by duplicating single nodes, thus balancing their influence with double nodes.

\begin{algorithm}
\small
\caption{STAMP-mag Graph Reconstruction}
\label{alg:stamp_mag}
\begin{algorithmic}[1]
\Require Undirected graph $\mathcal{G}(\mathbf{V}, \mathbf{E})$; adjacency matrix $A^{\mathcal{G}} \in \mathbb{R}^{N\times N}$; 
node time function $time: \mathbf{V}\to\mathbb{R}$; maximum time value $t_{\max}$; minimum time value $t_{\min}$
\Ensure New directed graph $\mathcal{G}'(\mathbf{V}, \mathbf{E}')$; new adjacency matrix $A^{\mathcal{G}'}$
\State $A^{\mathcal{G}'} \gets A^{\mathcal{G}}$
\For{$(u,v)\in\mathbf E$}
    \If{$time(u)=time(v)$ \textbf{or} 
    $|time(u)-time(v)|>\min(t_{\max}-time(v),\,time(v)-t_{\min})$}
        \State $A^{\mathcal{G}'}_{uv} \gets 2A^{\mathcal{G}}_{uv}$
    \EndIf
\EndFor
\end{algorithmic}
\end{algorithm}

\subsection{Comparison between \STAMP{} and STAMP-mag}

Although both methods reduce temporal first-moment shift, they correct different sources of aggregation bias. \STAMP{} directly balances the total contribution of source-timestamp buckets, whereas STAMP-mag corrects the unequal multiplicity of symmetric temporal offsets near the boundary of a finite chronological graph. Table~\ref{table:stamp_comparison} summarizes their differences.

\begin{table*}[t]
\small
\centering
\caption{Comparison between \STAMP{} and STAMP-mag.}
\label{table:stamp_comparison}

\vspace{4pt}

\begin{tabular}{p{0.13\textwidth}p{0.40\textwidth}p{0.40\textwidth}}
\hline
\textbf{Property}
& \textbf{\STAMP{}}
& \textbf{STAMP-mag} \\
\hline

Corrected imbalance
&
Unequal numbers of neighbors across source timestamps.
&
Unequal multiplicity between paired temporal offsets and one-sided offsets near chronological boundaries.
\\

Aggregation
&
Averages neighbors within each source-timestamp bucket and assigns equal total mass to every timestamp.
&
Assigns weight $2$ to same-time and one-sided neighbors and weight $1$ to paired neighbors, followed by normalization.
\\

Feature assumption
&
\begin{equation*}
    \mathbb E[X_u^{(0)}]
    =
    m_1(time(u)).
\end{equation*}
The feature mean depends only on the source timestamp.
&
\begin{equation*}
    \mathbb E[X_u^{(0)}]
    =
    \mu_1(y_u).
\end{equation*}
The feature mean depends only on the source label and is time-invariant.
\\

Graph assumption
&
Every target node has at least one neighbor at each timestamp:
\begin{equation*}
    |\mathcal N_v(\tilde t)|>0,
    \quad
    \forall v,\tilde t.
\end{equation*}
&
Neighbor connectivity depends symmetrically on temporal distance:
\begin{equation*}
    P_{yt}(\tilde y,\tilde t)
    =
    f(y,t)g(y,\tilde y,|\tilde t-t|).
\end{equation*}
\\

Guarantee
&
Exact first-moment invariance with respect to the target label and timestamp under the assumptions.
&
Approximate first-moment invariance with respect to the target timestamp under the assumptions.
\\

Recommended setting
&
Chronological graphs with broad cross-time connectivity and substantial imbalance among timestamp buckets.
&
Citation-like graphs with approximately symmetric temporal connectivity truncated near chronological boundaries.
\\

Main limitation
&
May be unsuitable when many target--timestamp buckets are empty or unreliable.
&
May be unsuitable when temporal connectivity is strongly directional or lacks distance symmetry.
\\

\hline
\end{tabular}

\end{table*}

Neither method uniformly dominates the other. \STAMP{} is more general with respect to temporal connectivity but requires sufficient timestamp coverage for each target node. STAMP-mag avoids timestamp-bucket estimation and is less restrictive regarding bucket coverage, but it relies on a substantially stronger structural assumption: temporal connectivity must be approximately symmetric as a function of distance from the target timestamp. Consequently, STAMP-mag should be viewed as a variant specialized to graphs exhibiting this structure, rather than as a replacement for \STAMP{}.

\section{Theoretical Analysis of \STAMP{}}
\label{sec4_theory_stamp}

We analyze how \STAMP{} affects the representation-level discrepancy between the training and test domains. Throughout this section, we assume the full timestamp-neighborhood coverage condition
\begin{equation}
    |\mathcal N_v(\tilde t)|>0,
    \qquad
    \forall v\in\mathbf V,\ \forall \tilde t\in\mathbf T.
    \label{eq:full_timestamp_coverage}
\end{equation}
This ensures that the \STAMP{} aggregation operator is well-defined. Note that the full timestamp-neighborhood coverage assumption is required for the theoretical analysis but not for practical implementation. When a node has no neighbors from a particular timestamp, the corresponding timestamp group is excluded from aggregation, and the weights are computed using the observed neighbors. Timestamp groups with few neighbors are handled in the same manner, although their contributions
may have higher variance.

Now, for a random variable $Z$, let $\mathsf{Law}(Z)$ denote its probability distribution. Denote ${tr}, {te}$ as the train, test domains, respectively, and let
\begin{equation}
    \mu_{tr}^{\mathrm{STAMP}}
    :=
    \mathsf{Law}(M_{{tr},\mathrm{STAMP}}^{(1)}),
    \qquad
    \mu_{te}^{\mathrm{STAMP}}
    :=
    \mathsf{Law}(M_{{te},\mathrm{STAMP}}^{(1)}).
\end{equation}

Assume that the initial node representations admit the first-order temporal decomposition
\begin{equation}
    X_w^{(0)}
    =
    m_1(\tilde t)+\xi_w,
    \qquad
    \forall w\in\mathcal N_v(\tilde t),
    \label{eq:first_order_temporal_decomp}
\end{equation}
where $m_1:\mathbf T\to\mathbb R^f$ is the timestamp-dependent mean and
\begin{equation}
    \mathbb E[\xi_w\mid time(w)=\tilde t]=0.
\end{equation}

\STAMP{} defines the first-layer message by
\begin{equation}
    M_{v,\mathrm{STAMP}}^{(1)}
    :=
    \frac{1}{|\mathbf T|}
    \sum_{\tilde t\in\mathbf T}
    \frac{1}{|\mathcal N_v(\tilde t)|}
    \sum_{w\in\mathcal N_v(\tilde t)}
    X_w^{(0)}.
    \label{eq:stamp_message}
\end{equation}
Under \eqref{eq:first_order_temporal_decomp}, its first moment is
\begin{equation}
    \mathbb E[M_{{tr},\mathrm{STAMP}}^{(1)}]
    =
    \mathbb E[M_{{te},\mathrm{STAMP}}^{(1)}]
    =
    m_{\mathrm{STAMP}},
    \qquad
    m_{\mathrm{STAMP}}
    :=
    \frac{1}{|\mathbf T|}
    \sum_{\tilde t\in\mathbf T}m_1(\tilde t).
    \label{eq:stamp_common_mean}
\end{equation}
Thus, \STAMP{} removes the first-moment train--test discrepancy caused by different neighborhood timestamp mixtures.

\begin{theorem}[Neighbor-size upper bound for \STAMP{}]
\label{thm:stamp_neighbor_size_w1_bound}
Suppose that the residuals in \eqref{eq:first_order_temporal_decomp} are pairwise uncorrelated in the sense that
\[
    \mathbb E\langle \xi_w,\xi_{w'}\rangle=0
    \qquad
    \text{for } w\neq w',
\]
and satisfy
\begin{equation}
    \mathbb E\|\xi_w\|^2\le \sigma_{\tilde t}^2
    \qquad
    \text{whenever } time(w)=\tilde t.
\end{equation}
If
\begin{equation}
    |\mathcal N_v(\tilde t)|\ge n_{\min},
    \qquad
    \forall v\in\mathbf V,\ \forall \tilde t\in\mathbf T,
    \label{eq:n_min_condition}
\end{equation}
and $\sigma_{\tilde t}^2\le\sigma^2$ for every $\tilde t\in\mathbf T$, then
\begin{equation}
    W_1(\mu_{tr}^{\mathrm{STAMP}},\mu_{te}^{\mathrm{STAMP}})
    \le
    \frac{2\sigma}{\sqrt{|\mathbf T|n_{\min}}}.
    \label{eq:stamp_neighbor_size_simple_bound}
\end{equation}
\end{theorem}

\begin{proof}
For a target node $v$, \eqref{eq:stamp_message} and
\eqref{eq:first_order_temporal_decomp} imply that the centered \STAMP{} residual is
\[
    \varepsilon_{v,\mathrm{STAMP}}
    :=
    M_{v,\mathrm{STAMP}}^{(1)}-m_{\mathrm{STAMP}}
    =
    \frac1{|\mathbf T|}
    \sum_{\tilde t\in\mathbf T}
    \frac1{|\mathcal N_v(\tilde t)|}
    \sum_{w\in\mathcal N_v(\tilde t)}
    \xi_w.
\]
Using pairwise uncorrelatedness and
$\mathbb E\|\xi_w\|^2\le\sigma_{\tilde t}^2$, we get
\[
    \mathbb E\|\varepsilon_{v,\mathrm{STAMP}}\|^2
    \le
    \frac1{|\mathbf T|^2}
    \sum_{\tilde t\in\mathbf T}
    \frac{\sigma_{\tilde t}^2}{|\mathcal N_v(\tilde t)|}
    \le
    \frac{\sigma^2}{|\mathbf T|n_{\min}} .
\]
Therefore,
\begin{align}
    W_1\left(
        \mu_{tr}^{\mathrm{STAMP}},
        \mu_{te}^{\mathrm{STAMP}}
    \right)
    &=
    W_1\left(
        \mathsf{Law}(\varepsilon_{{tr},{\mathrm{STAMP}}}),
        \mathsf{Law}(\varepsilon_{{te},{\mathrm{STAMP}}})
    \right) \\
    &\le
    W_1\left(
        \mathsf{Law}(\varepsilon_{{tr},{\mathrm{STAMP}}}),
        \delta_0
    \right)
    +
    W_1\left(
        \delta_0,
        \mathsf{Law}(\varepsilon_{{te},{\mathrm{STAMP}}})
    \right)\\
    &\le
    \mathbb E\|\varepsilon_{{tr},{\mathrm{STAMP}}}\|
    +
    \mathbb E\|\varepsilon_{{te},{\mathrm{STAMP}}}\| \\
    &\le
    \sqrt{\mathbb E\|\varepsilon_{{tr},{\mathrm{STAMP}}}\|^2}
    +
    \sqrt{\mathbb E\|\varepsilon_{{te},{\mathrm{STAMP}}}\|^2} \\
    &\le
    \frac{2\sigma}{\sqrt{|\mathbf T|n_{\min}}} .
\end{align}

\end{proof}

\begin{corollary}[Sufficient condition for \STAMP{} to reduce Wasserstein discrepancy]
\label{thm:stamp_w1_improvement}
Assume the conditions of Theorem~\ref{thm:stamp_neighbor_size_w1_bound}. Let
\begin{equation}
    \mu_{tr}^{\mathrm{agg}}
    :=
    \mathsf{Law}(M_{{tr},\mathrm{agg}}^{(1)}),
    \qquad
    \mu_{te}^{\mathrm{agg}}
    :=
    \mathsf{Law}(M_{{te},\mathrm{agg}}^{(1)}).
\end{equation}
and let
\begin{equation}
    \Delta_{\mathrm{agg}} := \mathbb E [ | M_{{tr}, \mathrm{agg}}^{(1)} - M_{{tr}, \mathrm{agg}}^{(1)} | ] .
\end{equation}
If
\begin{equation}
    |\mathbf T|n_{\min}
    >
    \left(\frac{2\sigma}{\Delta_{\mathrm{agg}}}\right)^2,
    \label{eq:main_w1_condition}
\end{equation}
then
\begin{equation}
    W_1(\mu_{tr}^{\mathrm{STAMP}},\mu_{te}^{\mathrm{STAMP}})
    <
    W_1(\mu_{tr}^{\mathrm{avg}},\mu_{te}^{\mathrm{avg}}).
\end{equation}
\end{corollary}

\begin{proof}
Since $\Delta_{\mathrm{agg}} \le W_1(\mu_{tr}^{\mathrm{avg}},\mu_{te}^{\mathrm{avg}})$,
Theorem \ref{thm:stamp_neighbor_size_w1_bound} gives
\[
    W_1(\mu_{tr}^{\mathrm{STAMP}},\mu_{te}^{\mathrm{STAMP}})
    <
    W_1(\mu_{tr}^{\mathrm{avg}},\mu_{te}^{\mathrm{avg}}).
\]
\end{proof}

\subsection{Class and temporal decomposition}

For a target node $v$, define the average representation of its neighbors at timestamp $\tilde t$ by
\begin{equation}
    \overline{X}_v^{(0)}(\tilde t)
    :=
    \frac{1}{|\mathcal N_v(\tilde t)|}
    \sum_{w\in\mathcal N_v(\tilde t)}
    X_w^{(0)}.
    \label{eq:timestamp_specific_average}
\end{equation}

Let $d\in\{tr,te\}$ denote the domain of the target node. We assume that the timestamp-specific neighborhood average admits the decomposition
\begin{equation}
    \overline{X}_v^{(0)}(\tilde t)
    =
    a_y
    +
    b(\tilde t)
    +
    r_{d,y}(\tilde t)
    +
    \xi_{v,\tilde t},
    \label{eq:class_temporal_decomposition}
\end{equation}
whenever the target node $v$ has label $y$ and belongs to domain $d$. Here, $a_y\in\mathbb R^f$ denotes the class-dependent component, $b(\tilde t)\in\mathbb R^f$ denotes the timestamp-dependent component, $r_{d,y}(\tilde t)\in\mathbb R^f$ denotes the approximation error, and $\xi_{v,\tilde t}\in\mathbb R^f$ denotes the centered residual.

We assume that
\begin{equation}
    \|r_{d,y}(\tilde t)\|
    \le
    \beta,
    \qquad
    \forall d\in\{tr,te\},\ 
    \forall y\in\mathbf Y,\ 
    \forall \tilde t\in\mathbf T.
    \label{eq:approximation_error_bound}
\end{equation}

We additionally assume that
\begin{equation}
    \mathbb E
    \left[
        \xi_{v,\tilde t}
        \mid
        y,d,
        \left\{
            |\mathcal N_v(s)|
        \right\}_{s\in\mathbf T}
    \right]
    =
    0.
    \label{eq:centered_residual}
\end{equation}

Suppose that the residuals corresponding to different timestamps are pairwise uncorrelated in the sense that
\begin{equation}
    \mathbb E
    \left[
        \left\langle
            \xi_{v,\tilde t},
            \xi_{v,\tilde s}
        \right\rangle
        \mid
        y,d,
        \left\{
            |\mathcal N_v(r)|
        \right\}_{r\in\mathbf T}
    \right]
    =
    0,
    \qquad
    \forall \tilde t\neq\tilde s.
    \label{eq:residual_uncorrelatedness}
\end{equation}

We further assume that the node-level fluctuations within each timestamp-specific neighborhood are sufficiently weakly dependent so that the second moment of their average decreases inversely with the neighborhood size:
\begin{equation}
    \mathbb E
    \left[
        \|\xi_{v,\tilde t}\|^2
        \mid
        y,d,
        \left\{
            |\mathcal N_v(s)|
        \right\}_{s\in\mathbf T}
    \right]
    \le
    \frac{\sigma^2}
    {|\mathcal N_v(\tilde t)|}.
    \label{eq:timestamp_residual_variance}
\end{equation}

The decomposition in \eqref{eq:class_temporal_decomposition} separates the class-dependent component $a_y$, which should be preserved for classification, from the timestamp-dependent component $b(\tilde t)$, which may induce a distribution discrepancy between the training and test domains.

By \eqref{eq:timestamp_specific_average}, the \STAMP{} message can be written as
\begin{equation}
    M_{v,\mathrm{STAMP}}^{(1)}
    =
    \frac{1}{|\mathbf T|}
    \sum_{\tilde t\in\mathbf T}
    \overline{X}_v^{(0)}(\tilde t).
    \label{eq:stamp_bucket_form}
\end{equation}

Substituting \eqref{eq:class_temporal_decomposition} into \eqref{eq:stamp_bucket_form}, we obtain
\begin{align}
    M_{v,\mathrm{STAMP}}^{(1)}
    ={}&
    a_y
    +
    \frac{1}{|\mathbf T|}
    \sum_{\tilde t\in\mathbf T}
    b(\tilde t)
    +
    \frac{1}{|\mathbf T|}
    \sum_{\tilde t\in\mathbf T}
    r_{d,y}(\tilde t)
    +
    \frac{1}{|\mathbf T|}
    \sum_{\tilde t\in\mathbf T}
    \xi_{v,\tilde t}.
    \label{eq:stamp_decomposition}
\end{align}

When $\beta=0$, we have
\begin{equation}
    \mathbb E
    \left[
        M_{v,\mathrm{STAMP}}^{(1)}
        \mid
        y,d
    \right]
    =
    a_y
    +
    \frac{1}{|\mathbf T|}
    \sum_{\tilde t\in\mathbf T}
    b(\tilde t).
    \label{eq:stamp_class_conditional_mean}
\end{equation}
The right-hand side of \eqref{eq:stamp_class_conditional_mean} does not depend on the domain $d$. Therefore,
\begin{equation}
    \mathbb E
    \left[
        M_{v,\mathrm{STAMP}}^{(1)}
        \mid
        y,tr
    \right]
    =
    \mathbb E
    \left[
        M_{v,\mathrm{STAMP}}^{(1)}
        \mid
        y,te
    \right].
    \label{eq:stamp_domain_mean_alignment}
\end{equation}

Moreover, for arbitrary labels $y,y'\in\mathbf Y$,
\begin{align}
    &\mathbb E
    \left[
        M_{v,\mathrm{STAMP}}^{(1)}
        \mid
        y,d
    \right]
    -
    \mathbb E
    \left[
        M_{v,\mathrm{STAMP}}^{(1)}
        \mid
        y',d
    \right]
    =
    a_y-a_{y'}.
    \label{eq:stamp_class_separation}
\end{align}

Therefore, under the decomposition \eqref{eq:class_temporal_decomposition}, \STAMP{} removes the temporal component which may differ between the training and test domains while preserving the first-order difference between different classes.

\subsection{Classification error bound}

Let $h = (h_y)_{y\in\mathbf Y}$ be a score-based classifier, where $h_y:\mathbb R^f\rightarrow\mathbb R$ is $L$-Lipschitz for every $y\in\mathbf Y$. Define the classification margin of $h$ at representation $x$ with label $y$ by
\begin{equation}
    \rho_h(x,y)
    :=
    h_y(x)
    -
    \max_{y'\neq y}h_{y'}(x).
    \label{eq:classification_margin}
\end{equation}

Since each $h_y$ is $L$-Lipschitz, the function $x \mapsto \rho_h(x,y)$ is $2L$-Lipschitz. For $\gamma>0$, define the margin loss by
\begin{equation}
    \ell_\gamma(z)
    :=
    \begin{cases}
        1, & z\le 0,\\
        1-\dfrac{z}{\gamma}, & 0<z<\gamma,\\
        0, & z\ge\gamma.
    \end{cases}
    \label{eq:margin_loss}
\end{equation}
For $d\in\{tr,te\}$, define the margin risk of \STAMP{} by
\begin{equation}
    R_d^\gamma(h)
    :=
    \mathbb E
    \left[
        \ell_\gamma
        \left(
            \rho_h
            \left(
                M_{v,\mathrm{STAMP}}^{(1)},y
            \right)
        \right)
        \mid
        d
    \right].
    \label{eq:margin_risk}
\end{equation}
Define the classification error by
\begin{equation}
    R_d(h)
    :=
    \mathbb P
    \left(
        \underset{y'\in\mathbf Y}{\arg\max}
        \ h_{y'}
        \left(
            M_{v,\mathrm{STAMP}}^{(1)}
        \right)
        \neq y
        \mid
        d
    \right).
    \label{eq:classification_error}
\end{equation}
For $d\in\{tr,te\}$, define the label distribution by
\begin{equation}
    \pi_d(y)
    :=
    \mathbb P(y\mid d).
    \label{eq:label_distribution}
\end{equation}
Define the discrepancy between the training and test label distributions by
\begin{equation}
    \Delta_{\mathbf Y}
    :=
    \frac{1}{2}
    \sum_{y\in\mathbf Y}
    \left|
        \pi_{tr}(y)
        -
        \pi_{te}(y)
    \right|.
    \label{eq:label_distribution_discrepancy}
\end{equation}

\begin{theorem}[Classification error bound for \STAMP{}]
\label{thm:stamp_classification_error}
Assume \eqref{eq:full_timestamp_coverage} and
\eqref{eq:class_temporal_decomposition}--\eqref{eq:timestamp_residual_variance}. Suppose that
\begin{equation}
    |\mathcal N_v(\tilde t)|
    \ge
    n_{\min},
    \qquad
    \forall v\in\mathbf V,\ 
    \forall \tilde t\in\mathbf T.
    \label{eq:classification_n_min}
\end{equation}
Then, for every $\gamma>0$,
\begin{equation}
    R_{te}(h)
    \le
    R_{tr}^{\gamma}(h)
    +
    \Delta_{\mathbf Y}
    +
    \frac{4L}{\gamma}
    \left(
        \beta
        +
        \frac{\sigma}
        {\sqrt{|\mathbf T|n_{\min}}}
    \right).
    \label{eq:stamp_classification_error_bound}
\end{equation}
\end{theorem}

\begin{proof}
For each $y\in\mathbf Y$, define
\begin{equation}
    c_y
    :=
    a_y
    +
    \frac{1}{|\mathbf T|}
    \sum_{\tilde t\in\mathbf T}
    b(\tilde t).
    \label{eq:stamp_class_center}
\end{equation}
It follows from \eqref{eq:stamp_decomposition} that
\begin{align}
    M_{v,\mathrm{STAMP}}^{(1)}-c_y
    ={}&
    \frac{1}{|\mathbf T|}
    \sum_{\tilde t\in\mathbf T}
    r_{d,y}(\tilde t)
    +
    \frac{1}{|\mathbf T|}
    \sum_{\tilde t\in\mathbf T}
    \xi_{v,\tilde t}.
    \label{eq:stamp_centered_message}
\end{align}
By \eqref{eq:approximation_error_bound},
\begin{equation}
    \left\|
        \frac{1}{|\mathbf T|}
        \sum_{\tilde t\in\mathbf T}
        r_{d,y}(\tilde t)
    \right\|
    \le
    \beta.
    \label{eq:stamp_approximation_term}
\end{equation}
By \eqref{eq:residual_uncorrelatedness},
\eqref{eq:timestamp_residual_variance}, and the law of total expectation,
\begin{align}
    \mathbb E
    \left[
        \left\|
            \frac{1}{|\mathbf T|}
            \sum_{\tilde t\in\mathbf T}
            \xi_{v,\tilde t}
        \right\|^2
        \mid
        y,d
    \right] &=
    \frac{1}{|\mathbf T|^2}
    \sum_{\tilde t\in\mathbf T}
    \mathbb E
    \left[
        \|\xi_{v,\tilde t}\|^2
        \mid
        y,d
    \right]
    \le
    \frac{\sigma^2}{|\mathbf T|^2}
    \mathbb E
    \left[
        \sum_{\tilde t\in\mathbf T}
        \frac{1}{|\mathcal N_v(\tilde t)|}
        \mid
        y,d
    \right]
    \le
    \frac{\sigma^2}
    {|\mathbf T|n_{\min}}.
    \label{eq:stamp_residual_second_moment}
\end{align}
By Jensen's inequality,
\begin{equation}
    \mathbb E
    \left[
        \left\|
            \frac{1}{|\mathbf T|}
            \sum_{\tilde t\in\mathbf T}
            \xi_{v,\tilde t}
        \right\|
        \mid
        y,d
    \right]
    \le
    \frac{\sigma}
    {\sqrt{|\mathbf T|n_{\min}}}.
    \label{eq:stamp_residual_first_moment}
\end{equation}

Therefore,
\begin{equation}
    \mathbb E
    \left[
        \left\|
            M_{v,\mathrm{STAMP}}^{(1)}
            -
            c_y
        \right\|
        \mid
        y,d
    \right]
    \le
    \beta
    +
    \frac{\sigma}
    {\sqrt{|\mathbf T|n_{\min}}}.
    \label{eq:stamp_distance_from_center}
\end{equation}
Since the function $x\mapsto\rho_h(x,y)$ is $2L$-Lipschitz and the function $\ell_\gamma$ is $1/\gamma$-Lipschitz, the function
\begin{equation}
    x
    \mapsto
    \ell_\gamma
    \left(
        \rho_h(x,y)
    \right)
\end{equation}
is $2L/\gamma$-Lipschitz.

For $d\in\{tr,te\}$ and $y\in\mathbf Y$, define
\begin{equation}
    L_{d,y}^{\gamma}(h)
    :=
    \mathbb E
    \left[
        \ell_\gamma
        \left(
            \rho_h
            \left(
                M_{v,\mathrm{STAMP}}^{(1)},y
            \right)
        \right)
        \mid
        y,d
    \right].
    \label{eq:conditional_margin_risk}
\end{equation}
It follows from \eqref{eq:stamp_distance_from_center} that
\begin{align}
    &\left|
        L_{d,y}^{\gamma}(h)
        -
        \ell_\gamma
        \left(
            \rho_h(c_y,y)
        \right)
    \right|
    \le
    \frac{2L}{\gamma}
    \left(
        \beta
        +
        \frac{\sigma}
        {\sqrt{|\mathbf T|n_{\min}}}
    \right).
    \label{eq:conditional_margin_center_bound}
\end{align}
Applying \eqref{eq:conditional_margin_center_bound} to both the training and test domains gives
\begin{equation}
    L_{te,y}^{\gamma}(h)
    \le
    L_{tr,y}^{\gamma}(h)
    +
    \frac{4L}{\gamma}
    \left(
        \beta
        +
        \frac{\sigma}
        {\sqrt{|\mathbf T|n_{\min}}}
    \right).
    \label{eq:conditional_margin_transfer}
\end{equation}
By the law of total expectation,
\begin{equation}
    R_d^\gamma(h)
    =
    \sum_{y\in\mathbf Y}
    \pi_d(y)
    L_{d,y}^{\gamma}(h).
    \label{eq:margin_risk_decomposition}
\end{equation}
Since $0\le L_{tr,y}^{\gamma}(h)\le1$, the definition of total variation distance gives
\begin{equation}
    \sum_{y\in\mathbf Y}
    \pi_{te}(y)
    L_{tr,y}^{\gamma}(h)
    \le
    \sum_{y\in\mathbf Y}
    \pi_{tr}(y)
    L_{tr,y}^{\gamma}(h)
    +
    \Delta_{\mathbf Y}.
    \label{eq:label_shift_bound}
\end{equation}

Combining \eqref{eq:conditional_margin_transfer},
\eqref{eq:margin_risk_decomposition}, and
\eqref{eq:label_shift_bound}, we obtain
\begin{equation}
    R_{te}^{\gamma}(h)
    \le
    R_{tr}^{\gamma}(h)
    +
    \Delta_{\mathbf Y}
    +
    \frac{4L}{\gamma}
    \left(
        \beta
        +
        \frac{\sigma}
        {\sqrt{|\mathbf T|n_{\min}}}
    \right).
    \label{eq:margin_risk_transfer}
\end{equation}
Finally, every misclassified node has a nonpositive classification margin. Therefore,
\begin{equation}
    R_{te}(h)
    \le
    R_{te}^{\gamma}(h).
\end{equation}
Combining this inequality with \eqref{eq:margin_risk_transfer} proves
\eqref{eq:stamp_classification_error_bound}.
\end{proof}

Theorem~\ref{thm:stamp_classification_error} provides a direct connection between \STAMP{} and the test classification error. In particular, the additional classification error caused by the temporal domain discrepancy is controlled by
\begin{equation}
    \frac{4L}{\gamma}
    \left(
        \beta
        +
        \frac{\sigma}
        {\sqrt{|\mathbf T|n_{\min}}}
    \right).
\end{equation}
Therefore, the bound becomes smaller when the decomposition in \eqref{eq:class_temporal_decomposition} is more accurate, when the timestamp-specific neighborhoods are larger, and when more timestamp buckets are aggregated.

\subsection{Optimality of \STAMP{} among first-order matching methods}

We next consider a general timestamp-weighted aggregation method of the form
\begin{equation}
    M_{v,q}^{(1)}
    :=
    \sum_{\tilde t\in\mathbf T}
    q_{\tilde t}
    \overline{X}_v^{(0)}(\tilde t),
    \label{eq:general_timestamp_aggregation}
\end{equation}
where
\begin{equation}
    q_{\tilde t}\ge0,
    \qquad
    \sum_{\tilde t\in\mathbf T}
    q_{\tilde t}=1.
    \label{eq:general_timestamp_weights}
\end{equation}

When $\beta=0$, we have
\begin{equation}
    \mathbb E
    \left[
        M_{v,q}^{(1)}
        \mid
        y,d
    \right]
    =
    a_y
    +
    \sum_{\tilde t\in\mathbf T}
    q_{\tilde t}
    b(\tilde t).
    \label{eq:general_aggregation_mean}
\end{equation}

Therefore, every fixed weight vector satisfying
\eqref{eq:general_timestamp_weights} aligns the training and test class-conditional first moments. However, different choices of $q_{\tilde t}$ result in different residual variances.

\begin{theorem}[Minimax optimality of \STAMP{}]
\label{thm:stamp_minimax_optimality}
Assume \eqref{eq:class_temporal_decomposition}--\eqref{eq:timestamp_residual_variance} with $\beta=0$, and suppose that
\begin{equation}
    |\mathcal N_v(\tilde t)|
    \ge
    n_{\min},
    \qquad
    \forall v\in\mathbf V,\ 
    \forall \tilde t\in\mathbf T.
\end{equation}
Then,
\begin{equation}
    \mathbb E
    \left[
        \left\|
            M_{v,q}^{(1)}
            -
            \mathbb E
            \left[
                M_{v,q}^{(1)}
                \mid
                y,d
            \right]
        \right\|^2
        \mid
        y,d
    \right]
    \le
    \frac{\sigma^2}{n_{\min}}
    \sum_{\tilde t\in\mathbf T}
    q_{\tilde t}^2.
    \label{eq:general_weight_variance_bound}
\end{equation}

Among all weights satisfying \eqref{eq:general_timestamp_weights}, the right-hand side of \eqref{eq:general_weight_variance_bound} is uniquely minimized by
\begin{equation}
    q_{\tilde t}
    =
    \frac{1}{|\mathbf T|},
    \qquad
    \forall\tilde t\in\mathbf T.
    \label{eq:uniform_optimal_weights}
\end{equation}
Therefore, \STAMP{} uniquely minimizes the worst-case residual variance among all fixed first-order matching aggregation methods when only the common variance upper bound in \eqref{eq:timestamp_residual_variance} is known.
\end{theorem}

\begin{proof}
It follows from \eqref{eq:general_timestamp_aggregation} and
\eqref{eq:class_temporal_decomposition} that
\begin{equation}
    M_{v,q}^{(1)}
    -
    \mathbb E
    \left[
        M_{v,q}^{(1)}
        \mid
        y,d
    \right]
    =
    \sum_{\tilde t\in\mathbf T}
    q_{\tilde t}
    \xi_{v,\tilde t}.
    \label{eq:general_aggregation_residual}
\end{equation}

By the pairwise uncorrelatedness of the residuals,
\begin{align}
    &\mathbb E
    \left[
        \left\|
            \sum_{\tilde t\in\mathbf T}
            q_{\tilde t}
            \xi_{v,\tilde t}
        \right\|^2
        \mid
        y,d
    \right]
    =
    \sum_{\tilde t\in\mathbf T}
    q_{\tilde t}^2
    \mathbb E
    \left[
        \|\xi_{v,\tilde t}\|^2
        \mid
        y,d
    \right]
    \le
    \frac{\sigma^2}{n_{\min}}
    \sum_{\tilde t\in\mathbf T}
    q_{\tilde t}^2.
\end{align}

By the Cauchy--Schwarz inequality,
\begin{equation}
    \left(
        \sum_{\tilde t\in\mathbf T}
        q_{\tilde t}
    \right)^2
    \le
    |\mathbf T|
    \sum_{\tilde t\in\mathbf T}
    q_{\tilde t}^2.
\end{equation}
Since
\begin{equation}
    \sum_{\tilde t\in\mathbf T}
    q_{\tilde t}
    =
    1,
\end{equation}
we obtain
\begin{equation}
    \sum_{\tilde t\in\mathbf T}
    q_{\tilde t}^2
    \ge
    \frac{1}{|\mathbf T|}.
    \label{eq:weight_lower_bound}
\end{equation}

Equality in \eqref{eq:weight_lower_bound} holds if and only if all weights are identical. Therefore, equality holds if and only if
\begin{equation}
    q_{\tilde t}
    =
    \frac{1}{|\mathbf T|},
    \qquad
    \forall\tilde t\in\mathbf T.
\end{equation}
This proves that the uniform weights used by \STAMP{} uniquely minimize the upper bound in
\eqref{eq:general_weight_variance_bound}.
\end{proof}

Theorem~\ref{thm:stamp_minimax_optimality} explains why \STAMP{} is preferable among the numerous methods which align the first moment. Although every fixed timestamp weighting method aligns the training and test first moments under \eqref{eq:class_temporal_decomposition}, \STAMP{} minimizes the worst-case stochastic error without requiring the estimation of timestamp-specific residual variances.

In particular, the classification error bound corresponding to the general aggregation method \eqref{eq:general_timestamp_aggregation} contains the stochastic term
\begin{equation}
    \frac{4L\sigma}
    {\gamma\sqrt{n_{\min}}}
    \left(
        \sum_{\tilde t\in\mathbf T}
        q_{\tilde t}^2
    \right)^{1/2},
    \label{eq:general_classification_stochastic_term}
\end{equation}
and this term is uniquely minimized by the \STAMP{} weights. Hence, \STAMP{} is the minimax choice when no timestamp is privileged and only the common upper bound in \eqref{eq:timestamp_residual_variance} is available.

\section{Experiments} \label{sec5_Experiments}

All experiments were conducted on Linux using Python and PyTorch. The real-world experiments used Python 3.10.9, PyTorch 2.6.0, DGL 2.5.0, and OGB 1.3.6, with
PyTorch Geometric 2.7.0 additionally used for Pokec. The Pokec and ogbn-arxiv experiments were conducted on CPU, while ogbn-mag used an NVIDIA Tesla V100 GPU
with 16 GB of memory. The synthetic experiments were conducted separately on CPU using Python 3.9.25, PyTorch 1.12.1, and DGL 0.9.1. To ensure reproducibility, we fixed random seeds and enabled deterministic computation settings. The synthetic experiments were repeated using three random seeds $0$ to $2$, while the real-world experiments were repeated using ten random seeds $0$ to $9$. The source code, detailed experimental configurations, software dependencies, and reproduction instructions are publicly available at \url{https://github.com/jooyoungpak/STAMP}.

\subsection{Synthetic Graph Experiments}

We evaluate \STAMP{} and STAMP-mag on two synthetic benchmarks designed to isolate distinct forms of temporal aggregation bias.

\paragraph{Temporal aggregation skew.}
The first benchmark is generated using a Temporal Stochastic Block Model (TSBM), in which the probability of a directed edge from source node $\tilde i$ to target node $i$ is
\begin{equation}
P(\tilde i \to i)
=
P_0[y_i,y_{\tilde i}]
\gamma[y_i,y_{\tilde i}]^{|t_i-t_{\tilde i}|}.
\end{equation}
The graph contains 3,000 nodes, five classes, and twelve balanced timestamp bins, each containing 250 nodes. Labels are sampled independently and uniformly, and 32-dimensional features are generated from weak class-dependent Gaussian distributions. For same-label pairs, we set $P_0=0.018$ and $\gamma=0.94$; for cross-label pairs, we set $P_0=0.22$ and $\gamma=0.025$. Cross-label edges are therefore abundant at the same timestamp, whereas same-label edges are sparse but persist across time. The resulting graph contains 163,276 directed edges, of which $82.15\%$ are same-time edges and $15.75\%$ are same-label edges.

We use timestamps $\{0,\ldots,6\}$ for training, $\{7,8\}$ for validation, and $\{9,10,11\}$ for testing, so the test domain comprises the most recent timestamp block. By giving each observed source-timestamp bucket equal total mass while preserving the ordinary mean-aggregation scale, \STAMP{} prevents the dominant same-time heterophilous bucket from overwhelming the persistent cross-time signal.

\paragraph{Temporal boundary asymmetry.}
The second benchmark is constructed to isolate the boundary structure targeted by STAMP-mag. It contains eight timestamps, three classes, and 100 nodes per timestamp--class pair, giving 2,400 nodes. Node features are generated as
\begin{equation}
x_v
=
\mu_{y_v}+\epsilon_v,
\qquad
\epsilon_v\sim\mathcal{N}(0,1.75^2I),
\end{equation}
where the class centers have scale $1.0$ and the feature dimension is 16. For a target node at timestamp $t$, define
\begin{equation}
r_t
=
\min(t-t_{\min},t_{\max}-t).
\end{equation}
Its incoming neighborhood contains 60 class-balanced same-time neighbors, six class-balanced neighbors from each of $t-\Delta$ and $t+\Delta$ for every $1\leq\Delta\leq r_t$, and 20 same-label neighbors in total, distributed over the unpaired timestamps satisfying
\begin{equation}
|t_u-t_v|>r_t.
\end{equation}
The paired temporal counts are exactly symmetric, while the one-sided boundary neighbors carry the label-informative signal.

The graph contains 235,200 directed edges with average in-degree 98.0. Same-time edges account for $61.22\%$ of all edges, while one-sided edges account for $20.41\%$ and are exclusively same-label. We use timestamps $\{0,\ldots,5\}$ for training, timestamp ${6}$ for validation, and timestamp ${7}$ for testing. Thus, the test set consists exactly of the most recent timestamp. STAMP-mag doubles the one-sided temporal contributions and normalizes the weights to preserve the ordinary mean-aggregation scale.

For both benchmarks, we evaluate SGC, GCN, GraphSAGE, and a linearized RevGAT using identical training settings. Table~\ref{synthetic} reports test accuracy at the epoch with the highest validation accuracy, averaged over seeds $\{0,1,2\}$. The reported deviations are population standard deviations across the three seeds.

\begin{table}[!ht] \label{synthetic}
\small
\centering
\caption{Test accuracy on the two synthetic benchmarks. The TSBM evaluates \STAMP{} under temporal bucket imbalance, while the boundary-asymmetric graph evaluates STAMP-mag when label-informative edges occur on only one temporal side. Results are averaged over seeds ${0,1,2}$.}
\label{table}

\vspace{4pt}
\setlength{\tabcolsep}{4pt}

\begin{tabular}{lccc|ccc}
\hline
& \multicolumn{3}{c|}{\textbf{Temporal Aggregation Skew}}
& \multicolumn{3}{c}{\textbf{Boundary Asymmetry}} \\
\textbf{Model}
& \textbf{Baseline}
& \textbf{\STAMP{}}
& \textbf{Gain}
& \textbf{Baseline}
& \textbf{STAMP-mag}
& \textbf{Gain} \\
\hline
SGC
& $0.2036 \pm 0.0220$
& $0.7573 \pm 0.0022$
& $+0.5538$
& $0.4089 \pm 0.0377$
& $0.9356 \pm 0.0318$
& $+0.5267$ \\

GCN
& $0.1951 \pm 0.0072$
& $0.8471 \pm 0.0025$
& $+0.6520$
& $0.6689 \pm 0.0113$
& $0.9189 \pm 0.1147$
& $+0.2500$ \\

GraphSAGE
& $0.3569 \pm 0.0242$
& $0.8924 \pm 0.0077$
& $+0.5356$
& $0.7789 \pm 0.0068$
& $0.9889 \pm 0.0042$
& $+0.2100$ \\

RevGAT$^{\dagger}$
& $0.2453 \pm 0.0214$
& $0.9062 \pm 0.0064$
& $+0.6609$
& $0.7300 \pm 0.0191$
& $0.9911 \pm 0.0103$
& $+0.2611$ \\
\hline
\end{tabular}

\vspace{2pt}
\begin{flushleft}
\footnotesize
$\dagger$For RevGAT \cite{RevGAT}, we use a linearized variant in which the graph attention layers are replaced by linear averaging layers.
\end{flushleft}

\end{table}

\subsection{Datasets} \label{subsec5_datasets}

\subsubsection{OGB} \label{subsubsec5_ogb}
The performance of \STAMP{} was evaluated on real-world datasets ogbn-mag and ogbn-arxiv from the Open Graph Benchmark (OGB) \cite{OGB}, both of which include temporal information with chronological splits. These datasets share the chronologically split property: the task is to predict labels of the most recent papers, framed as multi-class classification. The evaluation metric is accuracy, representing the proportion of correctly labeled nodes among all test nodes. For both datasets, nodes were divided into train (timestamp $\le 2017$), validation (timestamp $2018$), and test (timestamp $2019$) sets based on publication year. The objective was to classify $349$ labels for ogbn-mag and $40$ labels for ogbn-arxiv. Detailed node and edge statistics are provided in Table~\ref{table:ogbn node and edge}.

\paragraph{ogbn-mag.}
The ogbn-mag dataset is a large-scale heterogeneous chronological graph from OGB \cite{OGB}, comprising Paper, Author, Institution, and Field of study nodes, among which only paper nodes contain temporal information. Scalable methods therefore commonly use decoupled architectures that separate graph propagation from model training \cite{SeHGNN,RpHGNN}. We use the current SOTA HGAMLP-HOPE as the baseline method which follows this decoupled paradigm, while employing learned semantic aggregation rather than purely linear semantic aggregation.

\paragraph{ogbn-arxiv.}
For ogbn-arxiv, leading methods commonly combine pseudo-labeling and label propagation with nonlinear graph aggregation. We use the current SOTA SimTeG+TAPE+RevGAT as the baseline method. To apply \STAMP{} to RevGAT, we used a linearized variant in which graph attention layers are replaced by linear averaging layers, as in the synthetic graph experiments, and applied \STAMP{} to the resulting SimTeG+TAPE backbone.

\begin{table}[!ht]
\centering
\caption{Type and number of nodes and edges in experimental datasets.}
\setlength{\tabcolsep}{3pt}

\begin{tabular}{l|lccc|llll}
\hline
& \multicolumn{4}{c|}{\textbf{Nodes}} & \multicolumn{4}{c}{\textbf{Edges}} \\ \hline
\textbf{Dataset} & \textbf{Type} & \#Train & \#Valid & \#Test &
\textbf{Src} & \textbf{Edge type} & \textbf{Dst} & \#Edges \\ \hline

ogbn-arxiv
& Paper & 90,941 & 29,799 & 48,603 &
Paper & cites* & Paper & 1,166,243 \\ \hline

t.s.\ Pokec
& User & 1,218,296 & 297,038 & 117,306 &
User & connects to* & User & 30,622,117 \\ \hline

\multirow{4}{*}{ogbn-mag}
& Paper & 59,965 & 64,879 & 41,939 &
Author & affil.\ with & Inst. & 1,043,998 \\

& Author & 1,134,649 & 0 & 0 &
Author & writes & Paper & 7,145,660 \\

& Inst. & 8,740 & 0 & 0 &
Paper & cites* & Paper & 5,416,271 \\

& Field & 59,965 & 0 & 0 &
Paper & topic of & Field & 7,505,078 \\ \hline

\end{tabular}

\vspace{2pt}
\raggedright
\footnotesize{t.s.\ = temporal-split. * indicates edges connecting nodes with temporal information.}
\label{table:ogbn node and edge}
\end{table}

\subsection{Pokec} \label{subsubsec5_pokec}
We conducted an additional experiment on the Pokec dataset to evaluate generalizability. The Pokec dataset is a large-scale directed online social network collected from Pokec, a popular social networking service in Slovakia. Each node represents a user, and each directed edge represents a friendship relation. The publicly available SNAP version contains $1{,}632{,}803$ users and $30{,}622{,}564$ directed friendship edges. In addition to the graph structure, Pokec provides anonymized user-profile attributes such as gender, age, hobbies, interests, and education, written in Slovak \cite{nr}. Unlike ogbn-mag and ogbn-arxiv, Pokec is not originally provided with a standard chronological train/validation/test split. Therefore, to evaluate temporal out-of-distribution generalization in the same spirit as the OGB chronological splits, we constructed a chronological node split using the user registration year as the node timestamp.

Specifically, let $t_v$ denote the registration year of user node $v$. We first removed nodes whose registration year or target label was missing. Let
\begin{equation}
    t_{\max} := \max_{v \in \mathcal{V}_{\mathrm{labeled}}} t_v ,
\end{equation}
be the most recent registration year among labeled users after preprocessing. We then assigned nodes with $t_v \le t_{\max}-2$ to the training set, nodes with $t_v=t_{\max}-1$ to the validation set, and nodes with $t_v=t_{\max}$ to the test set. Thus, the model is trained only on historical users, validated on users from the year immediately before the test period, and evaluated on the most recent users. This follows the same chronological principle used for ogbn-mag and ogbn-arxiv, where recent nodes are held out to test robustness under temporal distribution shift.

Since Pokec does not provide edge timestamps, we used the full observed friendship graph in a transductive setting, while restricting supervision according to the chronological node split. That is, labels from validation and test nodes were not used during training, but their nodes and edges remained available for message passing, consistent with the transductive evaluation protocol of OGB node-classification benchmarks. The prediction task was formulated as user-profile classification. In our experiments, the target label was chosen from a profile attribute with sufficient coverage after preprocessing, and the same chronological split was used for all compared methods.

\subsection{Results on Real-World Graphs} \label{subsec6_realworld}

On the ogbn-arxiv dataset, applying \STAMP{} to SimTEG+TAPE+RevGAT improved performance by 0.25\% and achieved a new SOTA. On the temporal-split Pikec dataset, applying \STAMP{} to tuned GCN, which was the SOTA for the original Pokec, improved performance by 0.08\%. On the ogbn-mag dataset, applying STAMP-mag to HGAMLP-HOPE improved performance by 0.13\% and also achieved a new SOTA. The results are summarized in Table~\ref{table:results}.

\begin{table}[hbt!]
\centering
\small
\caption{Experimental results on real-world datasets. Best results are shown in bold. Results were averaged from seeds 0 to 9. }
\footnotesize
\begin{tabular}{llll}
\hline
Dataset & Model                       & Test Acc. $\pm$ Std. & Valid Acc. $\pm$ Std. \\ \hline
{ogbn-arxiv}
                          & \textbf{SimTEG+TAPE+RevGAT$^\dagger$+STAMP} & \textbf{0.7820 $\pm$ 0.0007} & 0.7867 $\pm$ 0.0009 \\
                          & SimTEG+TAPE+RevGAT (previous SOTA) & 0.7803 $\pm$ 0.0007 & 0.7846 $\pm$ 0.0004 \\
                          & SimTEG+TAPE+RevGAT (our environment) & 0.7795 $\pm$ 0.0015 & 0.7831 $\pm$ 0.0005 \\ \hline
{temporal-split Pokec}
                          & \textbf{GCN+STAMP} & \textbf{0.7676 $\pm$ 0.0016} & 0.7886 $\pm$ 0.0011 \\
                          & GCN (previous SOTA*) & 0.7668 $\pm$ 0.0017 & 0.7887 $\pm$ 0.0012 \\ \hline
{ogbn-mag}
                          & \textbf{HGAMLP-HOPE+STAMP-mag} & \textbf{0.5806 $\pm$ 0.0013} & 0.6002 $\pm$ 0.0007 \\
                          & HGAMLP-HOPE (previous SOTA) & 0.5801 $\pm$ 0.0012 & 0.6001 $\pm$ 0.0008 \\
                          & HGAMLP-HOPE (our environment) & 0.5793 $\pm$ 0.0015 & 0.5998 $\pm$ 0.0013 \\ \hline
\end{tabular}

\vspace{2pt}
\begin{flushleft}
\footnotesize
$\dagger$ To apply \STAMP{} to RevGAT \cite{RevGAT}, we use a linearized variant in which graph attention layers are replaced by linear averaging layers. \\
* As there is no official temporal-split Pokec dataset and leaderboard, this is the SOTA for the original Pokec dataset with same parameter settings, up to date.
\end{flushleft}
\label{table:results}
\end{table}

To further assess the performance improvements, we conducted two-sided paired $t$-tests using the results from ten random seeds $0$ to $9$. The improvements were statistically significant on ogbn-arxiv ($p<0.01$) and Pokec ($p=0.01$), but not on ogbn-mag ($p=0.11$).

\subsection{Computational Complexity}
\label{comp comp}

Both \STAMP{} and STAMP-mag are aggregation operators applied to a pre-existing graph learning method rather than standalone models. We therefore analyze their additional cost relative to the original method. Let $N=|\mathbf V|$, $E=|\mathbf E|$, and let $f$ denote the dimension of the initial node features. For \STAMP{}, let
\begin{equation}
B
:=
\left|
\{
(v,\tilde t)
\middle|
|\mathcal N_v(\tilde t)|>0
\}
\right| ,
\end{equation}
denote the number of nonempty target--timestamp buckets. Since every observed bucket contains at least one edge, $B\leq E$.

\begin{table}[!ht]
\small
\centering
\caption{Additional computational cost of adapting a pre-existing graph learning method with \STAMP{} or STAMP-mag. Both methods replace the ordinary first feature-propagation operation rather than adding propagation layers.}
\label{table:complexity}

\vspace{4pt}

\begin{tabular}{lcccc}
\hline
\textbf{Method}
& \textbf{Weight construction}
& \textbf{First propagation}
& \textbf{Additional time}
& \textbf{Additional memory} \\
\hline

\STAMP{}
& $\mathcal O(E)$
& $\mathcal O(Ef)$
& $\mathcal O(E)$
& $\mathcal O(B+E)$ \\

STAMP-mag
& $\mathcal O(E)$
& $\mathcal O(Ef)$
& $\mathcal O(E)$
& $\mathcal O(N+E)$ \\
\hline
\end{tabular}

\vspace{2pt}
\begin{flushleft}
\footnotesize
The first-propagation cost is not additional: it replaces the ordinary mean-aggregation step, which already requires $\mathcal O(Ef)$ time. Therefore, both methods add only a one-time $\mathcal O(E)$ weight-construction cost to the underlying method. The $\mathcal O(E)$ memory term is required only when the edge weights are stored explicitly.
\end{flushleft}

\end{table}

\paragraph{Additional cost of \STAMP{}.}
To construct the \STAMP{} coefficients, each edge is grouped according to its target node and source timestamp. Computing the $B$ observed bucket sizes and assigning the edge weights require $\mathcal O(E)$ time. The resulting weighted feature propagation requires $\mathcal O(Ef)$ time, which is the same asymptotic cost as the ordinary first mean-aggregation step that it replaces. Hence, adapting an existing method with \STAMP{} introduces only an additional one-time preprocessing cost of $ \mathcal O(E)$. The bucket counts require $\mathcal O(B)$ memory, while explicitly storing all edge weights requires an additional $\mathcal O(E)$ memory.

\paragraph{Additional cost of STAMP-mag.}
STAMP-mag determines whether each edge corresponds to a paired or one-sided temporal contribution using the timestamps of its source and target nodes. A single edge scan assigns the unnormalized weights, and a second scan computes the target-wise normalization. The weight-construction cost is therefore $\mathcal O(E)$. As with \STAMP{}, the weighted first propagation requires $\mathcal O(Ef)$ time and replaces the ordinary propagation already present in the underlying method. Thus, the additional time introduced by STAMP-mag is $\mathcal O(E)$. The method requires $\mathcal O(N)$ memory for the target-wise normalization factors and $\mathcal O(E)$ memory if the edge weights are stored explicitly.

Consequently, neither method changes the asymptotic training complexity of the underlying architecture. Both introduce a one-time linear preprocessing cost, while retaining the $\mathcal O(Ef)$ complexity of the original first feature-propagation operation. \STAMP{} additionally stores the observed target--timestamp buckets, whereas STAMP-mag requires only target-wise normalization values.

\section{Conclusions} \label{sec7_Conclusions}

In this work, we studied graph domain adaptation under chronological node splits. We proposed \STAMP{}, a timestamp-balanced feature propagation operator that aligns the first moment of initial messages across temporal domains, together with STAMP-mag, a specialized boundary-correction variant for ogbn-mag. Both methods can be integrated into existing architectures by modifying only the initial aggregation step, introducing little additional computational cost. Experiments on citation, social, and synthetic graphs demonstrate their scalability and effectiveness. In particular, \STAMP{} achieves SOTA performance on ogbn-arxiv, while STAMP-mag establishes SOTA performance on ogbn-mag. We further provide theoretical analyses characterizing their first-moment alignment properties under the stated assumptions.

Several limitations remain. The improvements over strong baselines are sometimes modest, and evaluation across a broader range of architectures and datasets would provide stronger evidence of generalizability. Moreover, the current analysis focuses on first-moment alignment but does not fully characterize alternative message passing methods which achieve first-moment alignment. Future work may consider higher-order distribution alignment, adaptive temporal weighting, and extensions to broader temporal graph learning settings.

\bmhead{Acknowledgements}

This work was supported by the National Research Foundation of Korea (NRF) grants funded by the Korea government (MSIT) (RS-2023-00222663 and 2022R1C1C1011888), the Institute of Information \& Communications Technology Planning \& Evaluation (IITP) grant funded by the Korea government (MSIT) (IITP-2024-RS-2024-00397085), and the Ministry of Education of the Republic of Korea and the National Research Foundation of Korea (NRF-2025S1A3A2A07022526).

\bmhead{Data Availability}

We used the public datasets in this research. The ogbn-arxiv and ogbn-mag datasets are available from the Open Graph Benchmark (OGB). The Pokec dataset is publicly available. Source code and experimental configurations are available at \url{https://github.com/jooyoungpak/STAMP}.

\bmhead{Declarations}

\textbf{Conflict of interest.} The authors declare that they have no conflict of interest.

\setlength{\bibsep}{0.25em plus 0.1em minus 0.05em}
\bibliography{sn-bibliography}

\begin{appendices}

\renewcommand{\thealgorithm}{\thesection.\arabic{algorithm}}
\counterwithin{algorithm}{section}

\newpage

\section{Assumption Verifications} \label{apdx:assumption verifications}

\subsection{Assumption of \STAMP{} in ogbn-arxiv and temporal-split Pokec}
In this section we provide an empirical verification of the feature-mean assumption of \STAMP{}. We empirically examine the assumption that the mean of the initial representation depends primarily on the source timestamp,
\begin{equation}
    \mathbb E[X_w^{(0)}]
    =
    m_1(\tilde t),
    \qquad
    w\in\mathcal N_v(\tilde y,\tilde t).
\end{equation}
Because exact equality cannot be established from finite samples, we evaluate whether timestamp provides a useful first-order approximation to the conditional feature mean. Let
\begin{equation}
    \widehat m(\tilde t)
    =
    \frac{1}{n_{\tilde t}}
    \sum_{w:time(w)=\tilde t}
    X_w^{(0)} .
\end{equation}
denote the empirical mean at timestamp $\tilde t$, and let $\widehat m$ denote the global feature mean. We measure the temporal first-moment effect by
\begin{equation}
    R_{\mathrm{time}}^2
    =
    \frac{
        \sum_{\tilde t\in\mathbf T}
        n_{\tilde t}
        \left\|
            \widehat m(\tilde t)-\widehat m
        \right\|_2^2
    }{
        \sum_{w\in\mathbf V}
        \left\|
            X_w^{(0)}-\widehat m
        \right\|_2^2
    }.
\end{equation}
To determine whether the feature mean depends only on timestamp, we additionally compute
\begin{equation}
    \Delta R_{\mathrm{label}\mid\mathrm{time}}^2
    =
    \frac{
        \sum_{\tilde t\in\mathbf T}
        \sum_{\tilde y\in\mathbf Y}
        n_{\tilde y,\tilde t}
        \left\|
            \widehat m(\tilde y,\tilde t)
            -
            \widehat m(\tilde t)
        \right\|_2^2
    }{
        \sum_{w\in\mathbf V}
        \left\|
            X_w^{(0)}-\widehat m
        \right\|_2^2
    },
\end{equation}
where $\widehat m(\tilde y,\tilde t)$ is the empirical mean conditioned jointly on label and timestamp. We report
\begin{equation}
    \rho_{\mathrm{label/time}}
    =
    \frac{
        \Delta R_{\mathrm{label}\mid\mathrm{time}}^2
    }{
        R_{\mathrm{time}}^2
    },
\end{equation}
which measures the remaining label-conditioned effect relative to the timestamp effect.

\paragraph{ogbn-arxiv.}
We use the exact TAPE RoBERTa representation employed as $X^{(0)}$ in our ogbn-arxiv experiments. Figure~\ref{fig:arxiv_feature_mean} shows clear temporal structure in the timestamp-conditioned means. Over 169,343 nodes and 2,484,941 directed edges, we obtain
\begin{equation}
    R_{\mathrm{time}}^2=0.02691,
\end{equation}
while edge weighting according to source-node occurrences in the propagation graph gives $R_{\mathrm{time,edge}}^2=0.03433$. Thus, the initial features exhibit a measurable temporal first-moment shift.

However, the residual label-conditioned effect is substantially larger:
\begin{equation}
    \Delta R_{\mathrm{label}\mid\mathrm{time}}^2
    =
    0.41996,
    \qquad
    \rho_{\mathrm{label/time}}
    =
    15.61.
\end{equation}
The corresponding edge-weighted residual is $0.42159$. Therefore, ogbn-arxiv supports the presence of temporal feature non-stationarity, but not the strict claim that the feature mean is determined only by timestamp. For this dataset, the assumption should be interpreted as isolating the temporal component of the first-moment shift rather than as an exact description of the conditional feature distribution.

\begin{figure}[t]
    \centering
    \includegraphics[width=\linewidth]
    {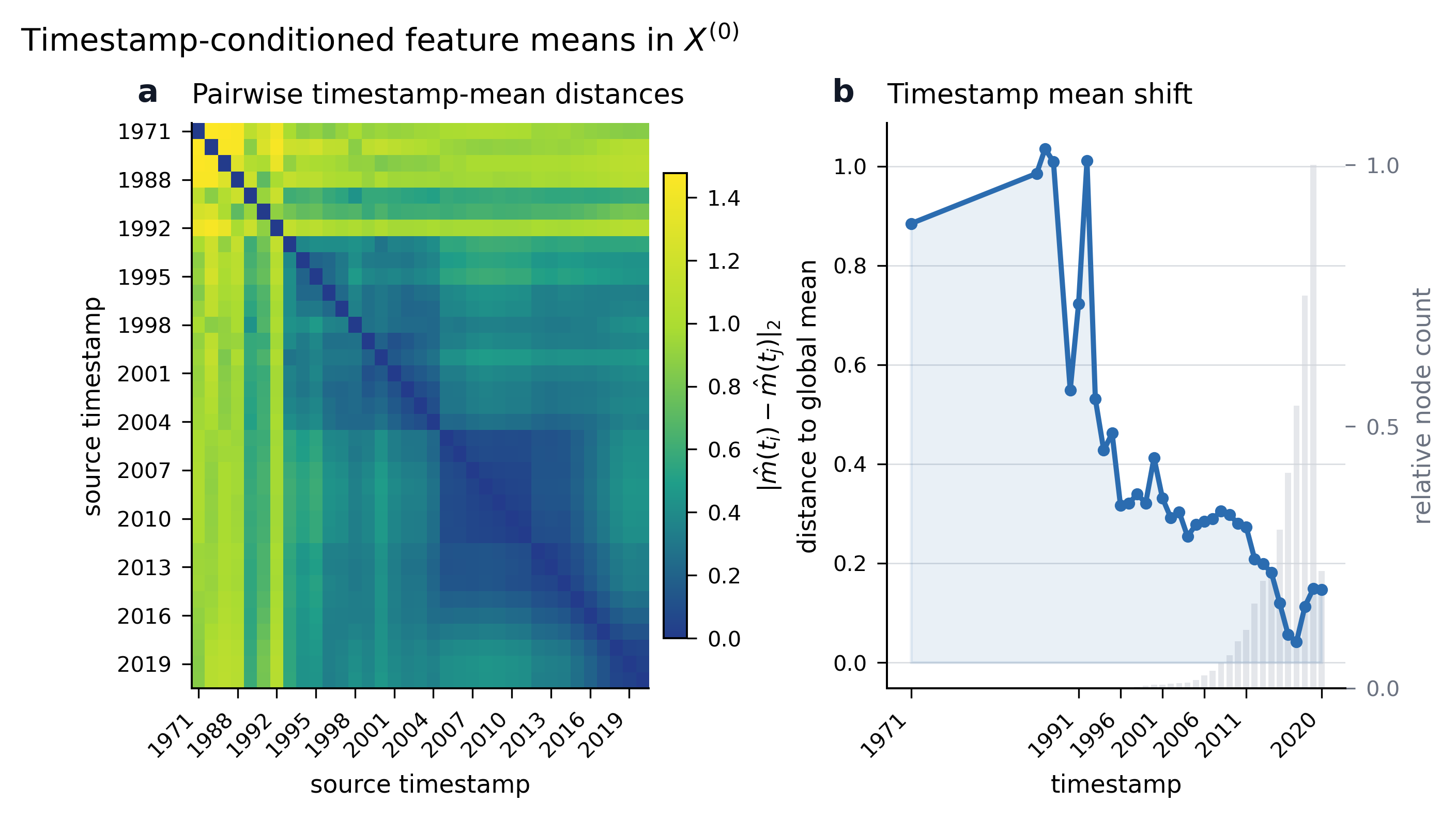}
    \caption{Timestamp-conditioned feature means on ogbn-arxiv using the exact TAPE RoBERTa representation employed as $X^{(0)}$. Panel \textbf{a} reports pairwise Euclidean distances between empirical timestamp means. Panel \textbf{b} reports the distance of each timestamp mean from the global mean, together with the relative number of nodes at that timestamp. The results reveal a measurable temporal first-moment shift.}
    \label{fig:arxiv_feature_mean}
\end{figure}

\paragraph{Temporal-split Pokec.}
For Pokec, we use the exact normalized initial features employed in the temporal-split GCN experiments. Figure~\ref{fig:pokec_feature_mean} shows a clear variation of the empirical feature mean across timestamps. Over 1,632,640 nodes and 46,236,731 directed edges, the temporal effect is
\begin{equation}
    R_{\mathrm{time}}^2
    =
    0.02266,
\end{equation}
with an edge-weighted value of $R_{\mathrm{time,edge}}^2=0.02038$.

In contrast to ogbn-arxiv, the remaining label-conditioned effect is considerably smaller:
\begin{equation}
    \Delta R_{\mathrm{label}\mid\mathrm{time}}^2
    =
    0.00320,
    \qquad
    \rho_{\mathrm{label/time}}
    =
    0.141.
\end{equation}
Under edge weighting, the residual effect is $0.00579$, corresponding to $\rho_{\mathrm{label/time,edge}}=0.284$. Hence, the timestamp effect is substantially larger than the additional label effect after conditioning on timestamp. These results support the timestamp-conditioned feature-mean assumption as a reasonable first-order approximation for the temporal-split Pokec experiment.

\begin{figure}[t]
    \centering
    \includegraphics[width=\linewidth]
    {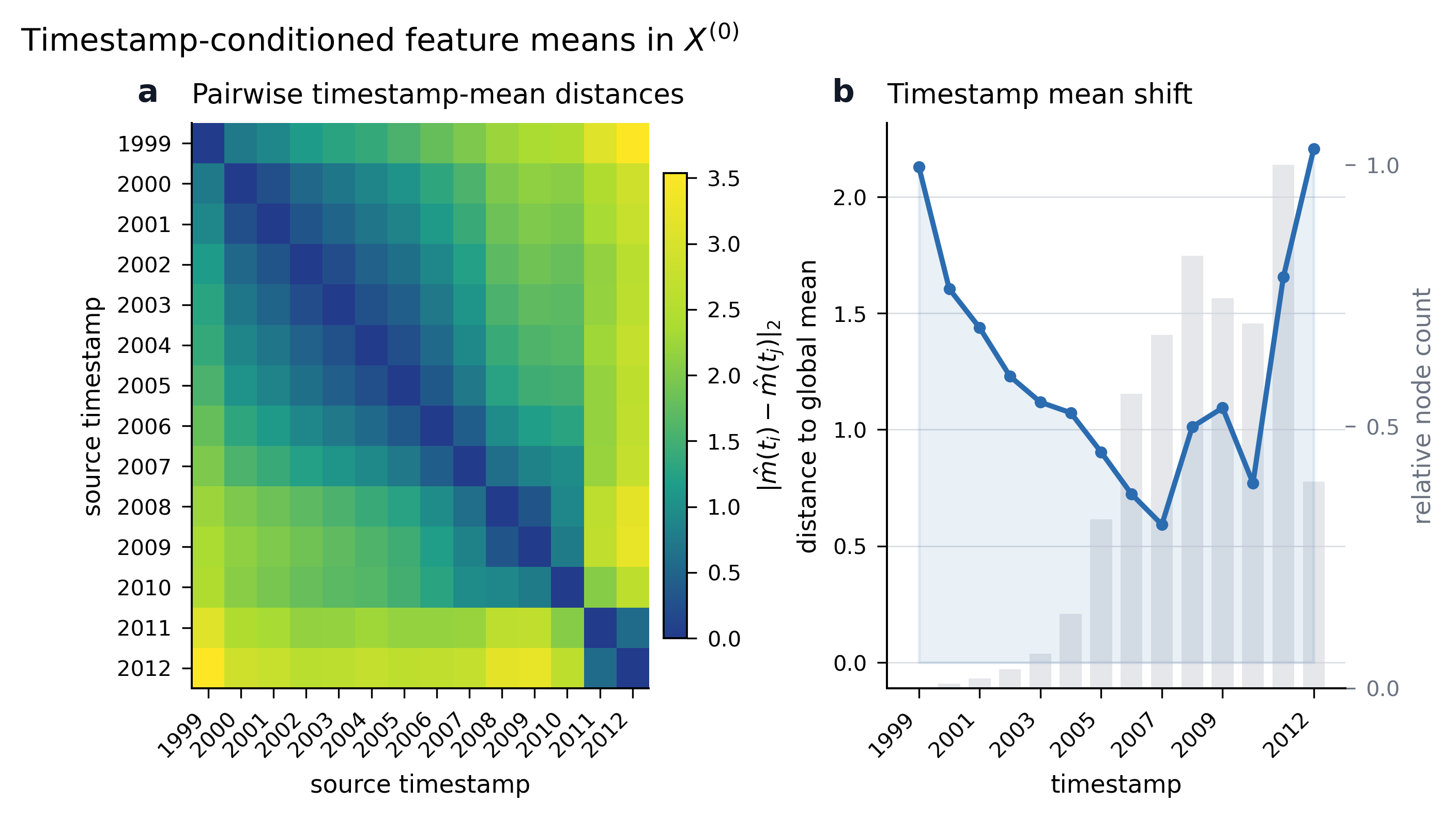}
    \caption{Timestamp-conditioned feature means on temporal-split Pokec using the exact normalized representation employed as $X^{(0)}$. Panel \textbf{a} reports pairwise Euclidean distances between empirical timestamp means. Panel \textbf{b} reports the distance of each timestamp mean from the global mean, together with the relative number of nodes at that timestamp. The timestamp effect is accompanied by comparatively small residual label dependence.}
    \label{fig:pokec_feature_mean}
\end{figure}

Overall, the validity of the approximation is dataset dependent. Temporal-split Pokec exhibits timestamp-dependent feature means with limited residual label dependence, directly supporting the assumption as a first-order model. In ogbn-arxiv, timestamp also contributes to the first-moment shift, but substantial label-conditioned variation remains.

\subsection{Assumptions of STAMP-mag in ogbn-mag}

We verify the two assumptions of STAMP-mag in this section.

Assumption 1 arises from empirical characteristics observed in the temporal graph ogbn-mag. To assess this assumption, we conducted a visual analysis on the ogbn-mag dataset. For fixed labels $y$ and $\tilde y$, we consider the estimated relative connectivity ${P}_{y,t}(\tilde y,\tilde t)$ as a function of $t$ and $\tilde t$. Under the assumption ${P}_{y,t}(\tilde y,\tilde t)=f(y,t)\,g\bigl(y,\tilde y,|\tilde t - t|\bigr)$, the graph of ${P}_{y, t} (\tilde y, \tilde t)$ for different $ t $ should have similar shapes, differing only by a scaling factor determined by $ f(y, t) $. Thus, after appropriate rescaling, curves corresponding to different $t$ values are expected to overlap.

Since $|\mathbf{Y}|=349$, visualizing this behavior for all label pairs $(y,\tilde y)$ is impractical. Instead, we selected several label pairs that are connected by a large number of edges and plotted their relative connectivity. Although ogbn-mag is a directed graph, we treat it as undirected when defining neighboring nodes.

In these plots, different colors represent different target node times $t$, while the x-axis corresponds to the relative time difference $\tilde t - t$ for neighboring nodes. Nodes with timestamps $t=2018$ and $t=2019$ were excluded since they correspond to validation and test data, respectively. The results shown in Figure \ref{fig:ERC} present the unscaled relative connectivity.

To summarize the global behavior in all label pairs, we further computed the weighted average relative connectivity for the cases $y=\tilde y$ and $y \neq \tilde y$. Specifically, for each $t$, we evaluated:

\begin{align}
P_{\text{Same label}}(t, \tilde t) = \frac{|\{(u, v) \in E \mid u,v \text{ have same label}, u\text{ has time } t, v \text{ has time } \tilde t\}|}{|\{(u, v) \in E \mid u,v \text{ have same label}\}|} ,
\end{align}

\begin{align}
P_{\text{Diff label}}(t, \tilde t) = \frac{|\{(u, v) \in E \mid u,v \text{ have different label}, u\text{ has time } t, v\text{ has time }\tilde t\}|}{|\{(u, v) \in E \mid u,v \text{ have different label}\}|} .
\end{align}

These quantities represent weighted averages over communities defined by each $(y,t)$ pair. Data corresponding to $t=2018$ and $t=2019$ were excluded, and no scaling adjustments were applied.

\begin{figure}[hbt!]
	\centering
	\includegraphics[width=0.99\textwidth]{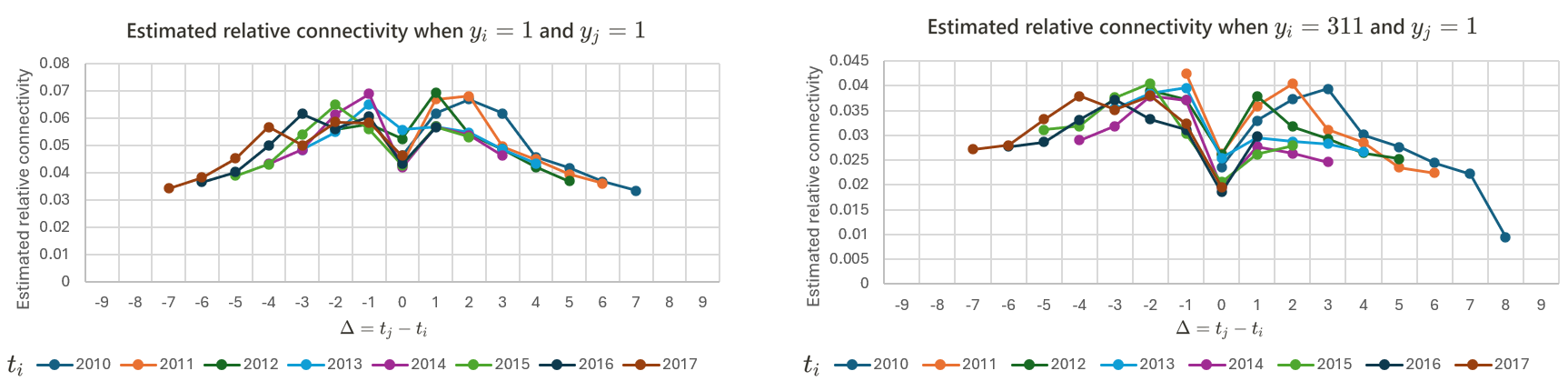}
	\caption{Estimated relative connectivity. [Left] when $y=1$ and $\tilde y=1$, [Right] when $y=311$ and $\tilde y=1$.}
	 \label{fig:ERC}
\end{figure}

\begin{figure}[hbt!]
	\centering
	\includegraphics[width=0.99\textwidth]{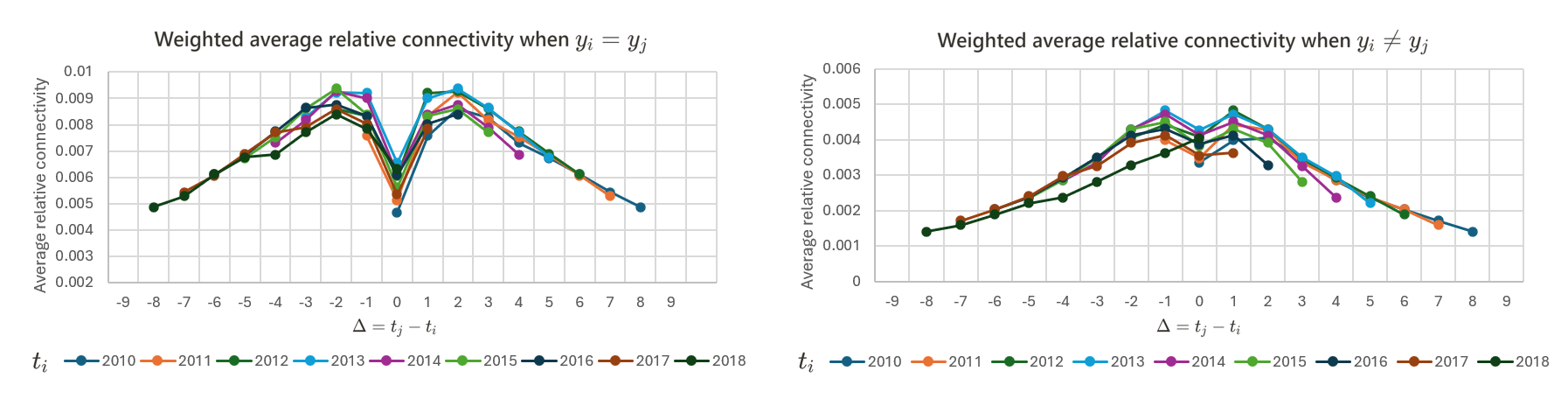}
	\caption{Weighted average relative connectivity. [Left] when $y=\tilde y$, [Right] when $y\neq \tilde y$.}
	 \label{fig:WARC}
\end{figure}

Figures~\ref{fig:ERC} and~\ref{fig:WARC} show that the connectivity curves corresponding to different values of $t$ exhibit similar shapes and approximate symmetry. These observations are consistent with Assumption 1, providing evidence for the separability assumption.

To verify assumption 2, we investigate the temporal stability of feature means. We computed the average node feature within each community, where a community is defined by a unique label-time pair. Our objective was to verify that the distance between mean features of nodes sharing the same label but belonging to different time periods is substantially smaller than the distance between mean features of nodes with different labels. Since features consist of high-dimensional embeddings, a direct comparison using standard norms may not adequately capture structural differences. Instead, we employed the unsupervised dimensionality reduction method t-SNE \cite{van2008visualizing} to project the community-level mean features onto a two-dimensional space. For the t-SNE analysis, we used a maximum of 1000 iterations, perplexity 30, and learning rate 50.

Figure \ref{fig:tsne} shows the mean feature for each community. Communities with identical labels are indicated using the same color. Consequently, each color represents $|\mathbf{T}|$ points, resulting in a total of $|\mathbf{Y}||\mathbf{T}|$ points in the left panel of Figure \ref{fig:tsne}.

Given that the number of labels is $|\mathbf{Y}|=349$, visualizing all labels simultaneously makes it difficult to identify structural patterns. The figure on the right panel of Figure \ref{fig:tsne} considers only the 15 labels with the most nodes, redrawing the graph for clarity.

\begin{figure}[hbt!]
	\centering
	\includegraphics[width=0.80\textwidth]{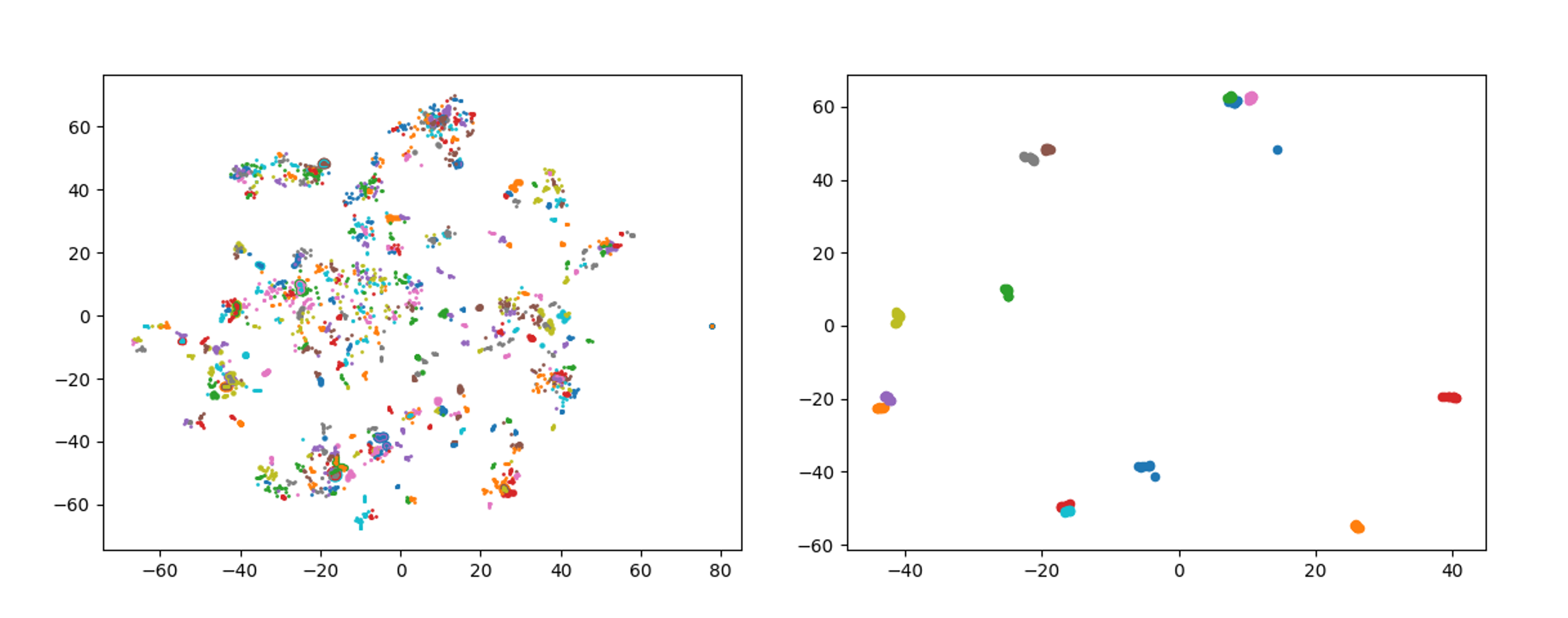}
	\caption{2D projection of each community's mean feature by t-SNE. Points corresponding to communities with the same label are represented in the same color. [Left] Plot for all 349 labels, [Right] Plot for 15 labels with the most nodes.}
	 \label{fig:tsne}
\end{figure}

The clustering structure associated with each color is clearly visible. This analysis is based on mean feature vectors, which indicates that communities sharing the same label but belonging to different time periods remain substantially closer to each other than communities with different labels, and therefore supporting assumption 2.

\end{appendices}

\end{document}